\documentclass[11pt,letterpaper]{article}
\usepackage[T1]{fontenc}
\usepackage[textwidth=6.1in,textheight=9in,centering]{geometry}
\usepackage{times}
\usepackage[round,authoryear]{natbib}
\usepackage{titlesec}
\usepackage{needspace,etoolbox}
\titleformat{\section}{\large\bfseries}{\thesection}{0.7em}{}
\titleformat{\subsection}{\normalsize\bfseries}{\thesubsection}{0.7em}{}
\titleformat{\subsubsection}{\normalsize\itshape}{\thesubsubsection}{0.7em}{}
\usepackage{amsmath,amssymb,amsthm,mathtools,bm,mathrsfs}
\usepackage{aliascnt}
\usepackage{booktabs,array,microtype,multirow}
\usepackage{graphicx}
\usepackage{xcolor}
\usepackage{enumitem}
\usepackage[colorlinks=true,linkcolor=blue!55!black,citecolor=blue!55!black,urlcolor=blue!55!black]{hyperref}
\usepackage[nameinlink,noabbrev]{cleveref}
\usepackage{url}
\usepackage{wrapfig}
\usepackage{float}
\usepackage{placeins}
\usepackage[font=small,labelfont=bf]{caption}
\newcommand{\R}{\mathbb R}
\newcommand{\Sph}{\mathbb S}
\newcommand{\dd}{\,\mathrm d}
\newcommand{\1}{\mathbf 1}
\newcommand{\relu}{\sigma}
\newcommand{\cL}{\mathcal L}
\newcommand{\cX}{\mathcal X}
\newcommand{\cP}{\mathcal P}

\newcommand{\proj}{\mathsf P}
\newcommand{\E}{\mathbb E}

\newcommand{\diver}{\operatorname{div}}

\newcommand{\Xfun}{\mathsf X}
\newcommand{\err}{\operatorname{Err}}

\newtheorem{theorem}{Theorem}[section]
\newaliascnt{proposition}{theorem}
\newtheorem{proposition}[proposition]{Proposition}
\aliascntresetthe{proposition}
\newaliascnt{lemma}{theorem}
\newtheorem{lemma}[lemma]{Lemma}
\aliascntresetthe{lemma}
\newaliascnt{corollary}{theorem}
\newtheorem{corollary}[corollary]{Corollary}
\aliascntresetthe{corollary}
\newaliascnt{assumption}{theorem}

\aliascntresetthe{assumption}
\theoremstyle{definition}
\newaliascnt{definition}{theorem}

\aliascntresetthe{definition}
\theoremstyle{remark}
\newaliascnt{remark}{theorem}
\newtheorem{remark}[remark]{Remark}
\aliascntresetthe{remark}
\BeforeBeginEnvironment{theorem}{\Needspace{7\baselineskip}}
\BeforeBeginEnvironment{lemma}{\Needspace{5\baselineskip}}
\BeforeBeginEnvironment{proposition}{\Needspace{5\baselineskip}}
\BeforeBeginEnvironment{corollary}{\Needspace{5\baselineskip}}
\crefname{theorem}{theorem}{theorems}
\crefname{proposition}{proposition}{propositions}
\crefname{lemma}{lemma}{lemmas}
\crefname{corollary}{corollary}{corollaries}
\crefname{assumption}{assumption}{assumptions}
\crefname{definition}{definition}{definitions}
\crefname{remark}{remark}{remarks}
\Crefname{theorem}{Theorem}{Theorems}
\Crefname{proposition}{Proposition}{Propositions}
\Crefname{lemma}{Lemma}{Lemmas}
\Crefname{corollary}{Corollary}{Corollaries}
\Crefname{assumption}{Assumption}{Assumptions}
\Crefname{definition}{Definition}{Definitions}
\Crefname{remark}{Remark}{Remarks}

\title{Hidden Gauge Controls Feature Specialization\\in ReLU Networks}
\author{Tongxi Wang\\School of Future Technology, Southeast University\\\texttt{tongxi\_wang@seu.edu.cn}}
\date{}
\hypersetup{pdftitle={Hidden Gauge Controls Feature Specialization in ReLU Networks},pdfauthor={Tongxi Wang},pdfsubject={Feature specialization and initialization in ReLU networks}}

\begin{document}
\begingroup
\centering
{\LARGE\bfseries Hidden Gauge Controls Feature Specialization\par
in ReLU Networks\par}
\vspace{1.5em}
{\large Tongxi Wang\par}
\vspace{0.4em}
School of Future Technology, Southeast University\par
\href{mailto:tongxi_wang@seu.edu.cn}{\texttt{tongxi\_wang@seu.edu.cn}}\par
\vspace{1.4em}
\endgroup
\begin{abstract}
The success of deep learning depends on learning useful representations, yet predicting how training organizes these representations across neurons remains difficult. In this work, we show that changing the scale of initial weights can determine which neurons learn a feature without altering any neuron's initial contribution. We construct ReLU networks with identical initial features and predictions that reach the same final predictions with different roles for their neurons. In one, all neurons share the learned feature. In the other, one neuron acquires it while every other neuron's contribution vanishes. The only change is the relative scale of each neuron's input and output weights. Our analysis explains how an initial learning advantage persists through convergence: as one neuron learns the target, it reduces the error driving the others and limits their subsequent adaptation. We prove this outcome in a nonlinear model under gradient flow and small-step gradient descent, and quantify how scale changes the speed and path of feature learning. Experiments verify the predicted dynamics and show that scale also changes feature assignment when two features compete. These results reveal how initialization can control the organization of a learned representation without changing what the network initially represents.
\end{abstract}

\raggedbottom
\section{Introduction}
Deep learning has made the acquisition of useful representations a central part of artificial intelligence \citep{lecun2015}. These representations emerge from the joint adaptation of many neurons, although much of a trained network can often be removed with little loss in predictive accuracy \citep{frankle2019,liu2019pruning}. This raises a natural question about how learning uses a network's capacity: what makes some neurons acquire useful features while others become redundant?

Empirical studies of pruning offer seemingly different answers about the importance of the starting weights. \citet{frankle2019} find sparse subnetworks that train successfully when reset to their original initial weights, whereas random reinitialization degrades their performance. In structured pruning experiments, however, \citet{liu2019pruning} find that training the retained architecture from random initialization can match or outperform fine-tuning the inherited weights. The usefulness of a retained architecture and that of its particular weights need not coincide.

These observations motivate an important distinction between recovering a network's accuracy and explaining how its components became useful. A reinitialized network may learn the task through a different internal representation. To understand which neurons carry that representation, we need to follow how initialization affects their learning. Starting weights determine both the features already represented and how those features adapt as the neurons reduce a shared prediction error. A neuron's eventual importance could therefore reflect an advantage in its initial feature or in its ability to adapt. To isolate the role of adaptation, we hold the architecture, data, and training rule fixed and ask whether neurons with identical initial features and contributions can still acquire different roles.

ReLU networks admit exactly this comparison. Multiplying a neuron's output weight by a positive factor and dividing its input weight by the same factor preserves its feature direction and its contribution to the network. This rescaling freedom is the gauge symmetry in our title. Gradient descent updates the two weights separately, so their relative scales affect how the neuron learns. \citet{kunin2024} show that this imbalance can accelerate feature learning.

Our main result is that changing only the relative weight scales can determine whether a feature is shared or learned by a single neuron. Our construction starts all neurons with identical features and contributions. When their relative scales also agree, they remain identical and share the learned feature equally. Giving one neuron a sufficiently large output weight relative to its input weight, and the others the opposite imbalance, instead makes that neuron learn the feature while every other coefficient tends to zero (\cref{fig:same-predictor}). Interestingly, both networks converge to the same prediction function. Thus, even identical initial and final predictions can conceal a different allocation of the learned feature across neurons.

\begin{figure}[tbp]
\centering
\includegraphics[width=0.95\linewidth]{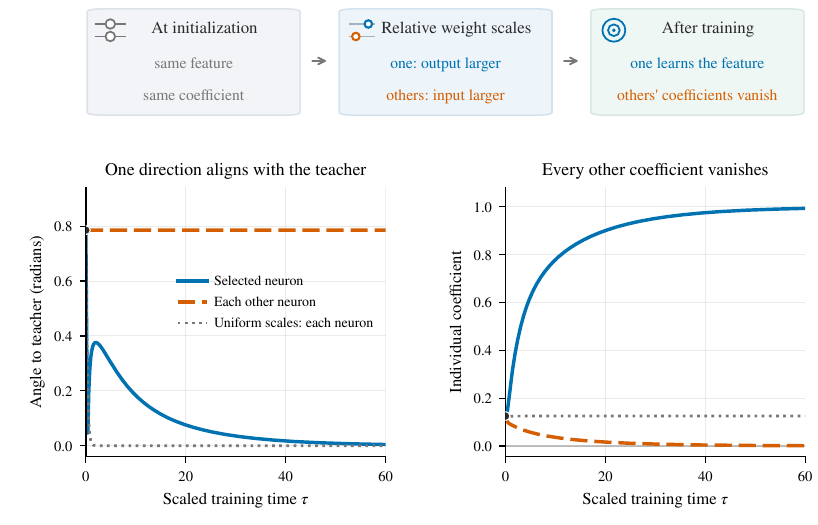}
\caption{Identical initial features, different roles after training. Eight students start with the same direction and coefficient $1/8$. With the same relative scale across all students (gray dotted), all learn the teacher and retain equal coefficients. With a larger output weight relative to the input weight in one student (blue), and the opposite imbalance in the other seven (orange dashed), only the blue direction aligns and every orange coefficient vanishes. Each cohort's curves coincide by symmetry. These are population gradient flows with imbalance magnitude $D=32$, where $D=|a^2-\|w\|^2|$ for output weight $a$ and input vector $w$. The coefficient is $q=a\|w\|$ and scaled time is $\tau=Dt$, where $t$ is training time.}
\label{fig:same-predictor}
\end{figure}

We establish this result in a Gaussian model with one teacher feature and duplicate student initializations, under the angle condition of \cref{thm:main-capture}. An early learning advantage alone does not guarantee this outcome, since slower neurons could continue to align and eventually share the feature. The proof shows why their remaining adaptation is limited. In the large-imbalance limit, the designated direction rotates toward the teacher while the other directions remain fixed. We prove convergence of this limiting system and then follow the trajectories at finite imbalance. As the selected student fits the target, the prediction error decays, limiting the entire subsequent rotation of the others. Their separation from the teacher forces their coefficients to vanish individually. Redundancy in this construction is therefore produced by the learning dynamics even though the neurons begin with the same feature.

Single-neuron analysis gives a quadratic separation in specialization times and inequivalent learning paths. Stability estimates extend finite-time selection to nearby initializations and the full convergence result to small-step population gradient descent. Experiments verify the predicted specialization and coefficient decay. With two competing features, they further show changes in assignment among nonduplicate students when only their scales change. Our results connect function-preserving choices at initialization to the organization of the learned representation.

\FloatBarrier

\raggedbottom
\section{Related work}
\paragraph{Scale and feature learning.} \citet{kunin2024} solve a minimal linear model and analyze how relative layer scales affect nonlinear feature learning. Their signed radial dynamics in Appendix C.2, Eq.~(131), give our coefficient and direction equations when both learning rates equal one.

\paragraph{Teacher--student networks.} Teacher--student analyses include population gradients and critical points \citep{tian2017}, specialization at finite width and dimension \citep{tian2020}, local convergence near over-realized solutions \citep{zhou2021}, and regularized recovery \citep{akiyama2021}. Overparameterization also affects the loss landscape and convergence speed \citep{safran2021,xudu2023}. ReLU identifiability results characterize equivalent parameters up to permutation and positive rescaling \citep{bonapellissier2022,bonapellissier2023}. Our local analysis uses a quantitative Gram-matrix condition to identify individual coefficients.

\paragraph{Symmetry and conservation.} Gradient flow in homogeneous networks preserves differences of squared layer norms \citep{du2018,marcotte2023}, and initialization dependence also appears in adaptive-knot models \citep{williams2019}. Path-lifting gives an intrinsic description of arbitrary-depth ReLU gradient flow for a dense set of initializations \citep{marcotte2025intrinsic}. Related optimization methods choose a better-conditioned parameterization \citep{lebeurrier2026}, encourage balance \citep{terin2026}, or use a geometry invariant to node rescaling \citep{neyshabur2015}. We keep Euclidean training fixed and use the conserved imbalance to compare trajectories from the same predictor.

\paragraph{Kernel and distributional dynamics.} Kernel analyses describe training near a fixed linearization \citep{jacot2018,chizat2019} and the evolution of features beyond that regime \citep{atanasov2022}. Infinite-width parameterizations \citep{yanghu2021}, mean-field dynamics \citep{mei2019}, and wide and deep linear models \citep{bordelon2025} further relate initialization to learning. Our analysis follows individual neurons through convergence at finite width.

\FloatBarrier

\section{Theory}
\label{sec:theory}
The main result concerns which neurons learn the target feature and which become redundant. To explain this outcome, we separate changes in a neuron's feature direction from changes in its coefficient. The single-neuron analysis identifies their relative speeds, and the multi-neuron analysis shows how interaction through a shared prediction error turns this difference into specialization.

\paragraph{Notation.} Throughout, $d\ge2$ is the input dimension, $\|\cdot\|$ is the Euclidean norm (the operator norm for matrices), and $\Sph^{d-1}$ is the unit sphere.  We write $[v]_+=\max\{v,0\}$ and $\angle(s,v)=\arccos(s^\top v)$ for unit vectors.  A dot denotes differentiation in training time $t$. We use $O(\cdot)$ for an upper bound and $\Theta(\cdot)$ for upper and lower bounds of the same order, with fixed parameters unless specified otherwise.

\subsection{Exact coefficient--direction dynamics}
Let $\Theta=((a_i,w_i))_{i=1}^m$ collect the output weights $a_i\in\R$ and input vectors $w_i\in\R^d$ of $m$ neurons.  For inputs $x$ with finite second moment and a square-integrable target $y(x)$, consider the predictor and population square loss
\begin{equation}
\label{eq:main-model}
 f_\Theta(x)=\sum_{i=1}^m a_i[w_i^\top x]_+,
 \qquad
 \cL(\Theta)=\frac12\E\!\left[(f_\Theta(x)-y(x))^2\right].
\end{equation}
Expectations are over the input distribution, which is assumed to give zero mass to the relevant ReLU switching hyperplanes. To describe the geometry of learning, we separate each neuron's coefficient, direction, and scale. For $w_i\ne0$, write
\begin{equation}
\label{eq:main-coordinates}
 r_i=\|w_i\|,
 \qquad s_i=\frac{w_i}{r_i},
 \qquad q_i=a_ir_i,
 \qquad \delta_i=a_i^2-r_i^2.
\end{equation}
The neuron's function is determined by $q_i$ and $s_i$, through $q_i[s_i^\top x]_+$. The remaining coordinate $\delta_i$ distinguishes parameterizations of that function. Let $R_\Theta=f_\Theta-y$ be the prediction residual, $I=I_d$ the identity matrix, and $\1_A$ the indicator of an event $A$. Define the residual averages and rate factors
\begin{equation}
\label{eq:main-mobilities}
\begin{aligned}
 B_\Theta(s)&=\E\!\left[R_\Theta(x)[s^\top x]_+\right],
 &\mathcal T_\Theta(s)&=\E\!\left[R_\Theta(x)\1_{\{s^\top x>0\}}x\right],\\
 \mu(q,\delta)&=\sqrt{\delta^2+4q^2},
 &\chi(q,\delta)&=\frac{2q}{\mu(q,\delta)-\delta}.
\end{aligned}
\end{equation}
Under Euclidean gradient flow $\dot\Theta=-\nabla_\Theta\cL$, and as long as $w_i\ne0$, these variables obey the exact system
\begin{equation}
\label{eq:main-reduction}
 \dot q_i=-\mu_i B_\Theta(s_i),
 \qquad
 \dot s_i=-\chi_i(I-s_is_i^\top)\mathcal T_\Theta(s_i),
 \qquad
 \dot\delta_i=0,
\end{equation}
where $\mu_i=\mu(q_i,\delta_i)$ and $\chi_i=\chi(q_i,\delta_i)$.  For $q_i\ne0$, this is the signed radial--directional system of \citet[Appendix C.2, Eq.~(131)]{kunin2024}: their $(\mu_i,\hat\beta_i)$ correspond to our $(q_i,s_i)$ when both layer learning rates equal one. The same coordinates separate the loss decrease into changes in coefficients and directions:
\begin{equation}
\label{eq:main-dissipation}
 -\dot\cL
 =\sum_i\left[
 \underbrace{\mu_iB_\Theta(s_i)^2}_{\text{coefficient change}}
 +\underbrace{\frac{\mu_i+\delta_i}{2}
 \bigl\|(I-s_is_i^\top)\mathcal T_\Theta(s_i)\bigr\|^2}_{\text{direction change}}
 \right].
\end{equation}
This expression separates fitting through coefficient adjustment from fitting through feature rotation. The conserved imbalance sets their relative rates, as derived in Appendix~\ref{app:marked}.

Notice that the coefficient rate is unchanged when the sign of the imbalance is reversed. At a shared $(q,s)$ with $q\ne0$ and the same residual, the two parameterizations with $\delta=\pm D$ have identical coefficient velocities, but their direction-update factors satisfy
\[
 \frac{\chi(q,+D)}{\chi(q,-D)}
 =\left(\frac{\sqrt{D^2+4q^2}+D}{2q}\right)^2
 =\Theta(D^2/q^2).
\]
Thus the same functional state can evolve at very different angular speeds. To understand the resulting trajectories, we must also follow the changing coefficients and residual.

For the Gaussian model below, let $\gamma_d=\mathcal N(0,I_d)$, with $\|f\|_{L^2(\gamma_d)}^2=\E_{x\sim\gamma_d}[f(x)^2]$, and fix a target direction $v_\star\in\Sph^{d-1}$.  We say that unit $i$ specializes to $v_\star$ if $s_i(t)\to v_\star$ and $\liminf_{t\to\infty}|q_i(t)|>0$.  We use coefficient-wise pruning to mean $q_j(t)\to0$ for every redundant unit $j$.

\subsection{Function-preserving scaling and specialization time}
\label{sec:single}
We begin with a single student, for which the residual is determined by its coefficient and alignment with the teacher. This gives a closed nonlinear system that lets us compare the full trajectories from a common initial function. We use Euclidean gradient flow throughout; \cref{prop:deep-node-mobility,rem:gauge-equivariant-updates,rem:weight-decay-marked} discuss rescaling in deeper networks, scale-equivariant updates, and weight decay.

Let $x\sim\mathcal N(0,I_d)$, $q_\star>0$, $\|s_\star\|=1$, $f_\star(x)=q_\star[s_\star^\top x]_+$, and $c=s^\top s_\star$.  Here $c$ is the alignment of the student direction with the teacher.  With the arc-cosine kernel
\[
 \kappa(c)=\frac{\sqrt{1-c^2}+(\pi-\arccos c)c}{2\pi},
\]
the residual averages depend only on $q$ and $c$, giving the closed system
\begin{equation}
\label{eq:main-single}
 \dot q=-\mu(q,\delta)\left(\frac q2-q_\star\kappa(c)\right),
 \qquad
 \dot c=\chi(q,\delta)q_\star\kappa'(c)(1-c^2),
 \qquad
 \dot\delta=0.
\end{equation}

The coefficient nullcline $q=2q_\star\kappa(c)$ is the set on which $\dot q=0$. We compare trajectories in the coefficient--alignment plane $(q,c)$. For a common initial state $(q_0,c_0)$ with $q_0>0$ and $-1<c_0<1$, let $(q_\delta,c_\delta)$ denote the solution with imbalance $\delta$. Given an alignment tolerance $0<\varepsilon<1-c_0$, define the specialization time
\[
 T_\delta(\varepsilon)=\inf\{t\ge0:c_\delta(t)\ge1-\varepsilon\}.
\]
The next theorem gives matching orders in both imbalance and accuracy. Explicit upper and lower constants are given in \cref{thm:global-specialization,thm:D2-gap}.

\begin{theorem}
\label{thm:main-gap}
For every fixed $\delta\in\mathbb R$, $(q_\delta(t),c_\delta(t))\to(q_\star,1)$. For all sufficiently large $D$, uniformly over $0<\varepsilon\le(1-c_0)/2$,
\begin{equation}
\label{eq:main-gap-plus}
 T_{+D}(\varepsilon)=\Theta\!\left(D^{-1}\log(1/\varepsilon)\right),\qquad
 T_{-D}(\varepsilon)=\Theta\!\left(D\log(1/\varepsilon)\right).
\end{equation}
Moreover, uniformly over the full range $0<\varepsilon<1-c_0$,
\begin{equation}
\label{eq:main-gap}
 T_{-D}(\varepsilon)/T_{+D}(\varepsilon)=\Theta(D^2).
\end{equation}
The threshold on $D$ and the constants in these two-sided bounds depend only on $q_0,c_0,q_\star$.
\end{theorem}
\begin{corollary}
\label{cor:main-no-clock}
For every $D>0$, $q_0>0$, and $-1<c_0<1$, the trajectories of \eqref{eq:main-single} with $\delta=+D$ and $\delta=-D$ and common initial state $(q_0,c_0)$ cannot agree on any nontrivial interval after a strictly increasing $C^1$ reparameterization of time.  Away from $q/2=q_\star\kappa(c)$ the two vector fields are non-collinear.  On this nullcline they have the same tangent direction, but their phase curves have different curvature.
\end{corollary}
The two imbalances therefore yield the same learned feature through different sequences of learning, as illustrated in \cref{fig:single-paths}. With negative imbalance, the coefficient first relaxes while the direction remains nearly fixed; positive imbalance allows rapid alignment. This difference persists on the coefficient nullcline: a common phase curve would have to remain on the nullcline with constant coefficient, which is impossible (\cref{cor:not-time-rescaling}).

\begin{figure}[tbp]
\centering
\includegraphics[width=0.95\linewidth]{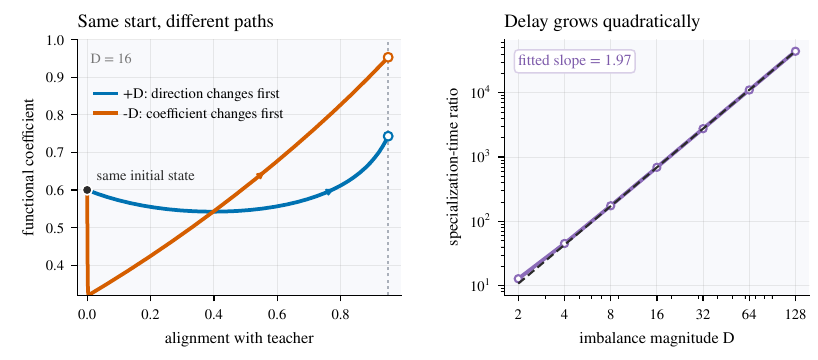}
\caption{Opposite imbalances change the path and speed of single-student learning. Left: coefficient--alignment trajectories from a common initial state, at $D=16$. Right: the ratio of specialization times grows quadratically with $D$. The dashed line has slope two.}
\label{fig:single-paths}
\end{figure}

\subsection{Feature assignment among duplicate neurons}
\label{sec:duplicates}
The single-student result establishes a difference in adaptation rates. We now show how this difference produces the allocation in \cref{fig:same-predictor} when students interact through the same residual. We initialize them with identical functional contributions and change only their parameter scales. Let the teacher be $f_\star(x)=[u^\top x]_+$ with $\|u\|=1$, and initialize $m\ge2$ students with the same coefficient and feature direction:
\[
 q_i(0)=\frac1m,
 \qquad s_i(0)=s_0,
 \qquad 0<\phi=\angle(s_0,u)\le\pi/3.
\]
For uniform constants, fix any compact interval $I_\phi\subset(0,\pi/3]$
containing $\phi$.  Set $\delta_k=+D$ for one designated index $k$ and $\delta_j=-D$ for every $j\ne k$.
For comparison, if all imbalances equal a common $\delta$, symmetry keeps the students identical. Their total coefficient $Q=mq_i$ and alignment $c=s_i^\top u$ obey \eqref{eq:main-single} with imbalance $m\delta$ and teacher coefficient one. Thus \cref{thm:main-gap} gives $q_i\to1/m$ and $s_i\to u$ for every student. The mixed scale assignment changes this shared representation into the following one.

Write the contribution of the redundant units as
\[
 f_{\rm red}(x,t):=\sum_{j\ne k}q_j(t)[s_j(t)^\top x]_+.
\]

\begin{theorem}
\label{thm:main-capture}
Consider Euclidean gradient flow on the population loss \eqref{eq:main-model} with $x\sim\mathcal N(0,I_d)$, $y=f_\star$, and the initialization above.  For every fixed $m\ge2$ and $\phi\in I_\phi$, there are $D_0(m),C_m,c_m>0$, independent of $D$, such that for every $D\ge D_0(m)$,
\[
 q_k(t)\to1,\qquad s_k(t)\to u,\qquad q_j(t)\to0\quad(j\ne k).
\]
There is a time $t_{\rm cap}\le C_mD^{-1}$ after which the loss decays exponentially:
\[
 \cL(t)\le \cL(t_{\rm cap})e^{-c_mD(t-t_{\rm cap})},\qquad t\ge t_{\rm cap},
\]
and every redundant direction satisfies
\[
 \int_{t_{\rm cap}}^\infty\|\dot s_j(t)\|\,\dd t\le C_mD^{-2}
 \qquad(j\ne k).
\]
Thus unit $k$ is the unique student that specializes to the teacher, and the remaining units are coefficient-wise pruned.  Permuting the unique $+D$ imbalance permutes the selected index while leaving the initial predictor unchanged.
\end{theorem}
Coefficient-wise pruning has a concrete meaning in the original parameters. For a redundant unit with conserved imbalance $-D$, the relation $q_j=a_j\|w_j\|$ gives
\[
 a_j^2=\frac{\sqrt{D^2+4q_j^2}-D}{2},
 \qquad \|w_j\|^2=\frac{\sqrt{D^2+4q_j^2}+D}{2}.
\]
As $q_j\to0$, its output weight tends to zero while its input norm tends to $\sqrt D$. The unit loses its contribution to the prediction even though its input vector need not vanish. This distinction explains why coefficient decay and feature alignment are measured separately in the experiments.

\subsubsection{Why the initial advantage persists}
Thus the initial scale determines the limiting representation, including the disappearance of every redundant contribution. To understand how this occurs, we study the fast time $\tau=Dt$ (\cref{thm:global-selection}). Let $q=q_k$, let $P=\sum_{j\ne k}q_j$ be the total redundant coefficient, and let $\theta$ and $\psi$ be the signed angles of the selected and common redundant directions relative to $u$. The teacher plane is spanned by $u$ and $s_0$, with orientation chosen so that $s_0$ has angle $\phi$. In the limit, $\psi=\phi$ remains fixed. If $s_\theta$ is the unit vector at angle $\theta$ in this plane, the limiting predictor is $q[s_\theta^\top x]_++P[s_0^\top x]_+$. Its population loss is
\[
 E(q,P,\theta)=\frac12\bigl\|q[s_\theta^\top\cdot]_+
                  +P[s_0^\top\cdot]_+-f_\star\bigr\|_{L^2(\gamma_d)}^2.
\]
Writing primes for $\tau$ derivatives, the limiting dynamics are
\begin{equation}
\label{eq:main-fast}
 q'=-\partial_qE,\qquad P'=-(m-1)\partial_PE,\qquad
 \theta'=-q^{-2}\partial_\theta E,
\end{equation}
from $(q,P,\theta)=(1/m,(m-1)/m,\phi)$. Notice that all redundant students enter through their total coefficient $P$. Their number changes the rate of its adjustment, while the loss landscape remains the same. This reduction separates two issues that are difficult to distinguish in the full network: whether the selected feature reaches the teacher, and whether the redundant units can recover a nonzero contribution later. The limiting flow resolves the first issue. Controlling the remaining motion at finite imbalance resolves the second.

We first show that the coefficient-weighted direction $z=qs_\theta$ develops a component toward the teacher, keeping $z$ away from zero. Kernel inequalities then exclude every stationary state except $(1,0,0)$ in an invariant region containing the initialization. The limiting flow converges to this state, leaving the redundant direction at angle $\phi>0$.

At finite imbalance, the redundant directions also move. Contraction of their total amplitude and the projection bound control the resulting trajectory error, including the initial period when the selected coefficient is small. This places the finite-$D$ trajectory near $(1,0,0)$ for $D\ge D_0(m)=Am^{\nu/2}$, where $A>0$ and $\nu>3/2$ depend only on $I_\phi$ (\cref{prop:mdep}).

Near this state, a positive local Gram matrix controls the selected coefficient, its angle, and the total redundant coefficient through the loss. It also gives exponential loss decay while the trajectory remains in the neighborhood. To justify using this local estimate for all later times, we arrange for the entry loss to lie below the boundary loss for these coordinates. Decreasing loss then prevents a first exit through that boundary, while the bound on total redundant rotation prevents exit through the angular boundary. The trajectory stays in the neighborhood and converges. Symmetry yields $q_j=P/(m-1)\to0$ for each redundant student. These two controls are what turn a finite-time comparison with the limiting flow into convergence at finite imbalance.

\subsection{Robustness and discrete gradient descent}
\label{sec:robustness}
We next examine how the selection mechanism changes when students begin at nearby states or training proceeds in discrete steps. With one positive scale imbalance and $m-1$ negative ones, each with magnitude near $D$, measure departure from duplicate initialization by
\begin{equation}
\label{eq:main-perturbation}
 \Delta_0=\max_i\left(\left|q_i(0)-\frac1m\right|+\|s_i(0)-s_0\|\right)
 +\max_i\left|\frac{|\delta_i(0)|}{D}-1\right|.
\end{equation}
Choose a small selection tolerance $r_\star>0$ relative to the initial angle and the local convergence neighborhood. Let $T_\star$ bound the fast time needed by the limiting flow to enter a smaller neighborhood, and let $D_0(m)=Am^{\nu/2}$ be the sufficient imbalance threshold above. Appendix~\ref{app:robust-discrete} fixes these constants before \cref{prop:mdep}. For the state $Z=((q_\ell,s_\ell))_{\ell=1}^m$, write $f_{-i}(x;Z)=\sum_{j\ne i}q_j[s_j^\top x]_+$ for the output of all other students. Unit $i$ is selected when its state enters
\[
 \mathcal W_{\rm sel}^{(i)}:=\left\{Z:
 q_i\ge1-r_\star,\quad
 \angle(s_i,u)\le r_\star,\quad
 \|f_{-i}(\cdot;Z)\|_{L^2(\gamma_d)}\le r_\star
 \right\}.
\]

\begin{theorem}
\label{thm:main-robust}
There is $a>0$, depending only on $I_\phi$, such that, with $\nu$ as above, for every $m\ge2$ and $\phi\in I_\phi$, if $\Delta_0\le\varepsilon_m:=am^{-\nu}$ and $D\ge D_0(m)$, then the first entry time
\[
 t_{\rm sel}:=\inf\{t\ge0:Z(t)\in\mathcal W_{\rm sel}^{(k)}\}
\]
satisfies $t_{\rm sel}\le(T_\star+1)/D$, while $Z(t)\notin\mathcal W_{\rm sel}^{(i)}$ for every $i\ne k$ and $0\le t\le t_{\rm sel}$.  At $t_{\rm sel}$, every unit with negative imbalance satisfies $\angle(s_j(t_{\rm sel}),u)>r_\star$.
\end{theorem}
The designated neuron therefore reaches the teacher neighborhood first. The proof controls the total redundant amplitude and each neuron's deviation from the mean separately (\cref{thm:robust-selection}). Subsequent convergence and individual coefficient decay follow under the local Gram-matrix bound and entry-loss condition of \cref{thm:robust-locking}.

Gradient descent introduces a different perturbation because each step changes the imbalance. We return to duplicate initialization with $\phi\in I_\phi$, set $\delta_k=+D$, $\delta_j=-D$ for $j\ne k$, and use gradient descent with learning rate $\eta=h/D$, where $h>0$ is the step in fast time.  Write $Y_n=(q_n,P_n,\theta_n,\psi_n)$ for the symmetric state after $n$ iterations with $\cL_n=\cL(\Theta_n)$ and $\delta_{i,n}=a_{i,n}^2-\|w_{i,n}\|^2$.  Let $T_c=T_\star+1$ and $Y_\infty(\tau)=(q_\infty(\tau),P_\infty(\tau),\theta_\infty(\tau),\phi)$, where the first three coordinates solve \eqref{eq:weighted-fast}.  Let $\mathcal V_{\rm in}$ be an inner neighborhood of the selected solution, with a strict loss gap to the boundary of a larger neighborhood.  The construction in \eqref{eq:discrete-neighborhoods} also bounds the original parameter scales away from zero.

\begin{theorem}
\label{thm:main-discrete}
For every fixed $m\ge2$ and $\phi\in I_\phi$, there are $h_0(m),\widetilde D_0(m),C_m,c_m>0$ such that, for all $0<h\le h_0(m)$ and $D\ge\widetilde D_0(m)$, full-batch gradient descent on the population loss converges to the selected representation: $q_n\to1$, $P_n\to0$, and $\theta_n\to0$. In particular, each redundant coefficient $q_{j,n}=P_n/(m-1)$ tends to zero. Up to fast time $T_c$, the trajectory satisfies
\[
 \max_{nh\le T_c}\|Y_n-Y_\infty(nh)\|
 \le C_m(h+D^{-2}).
\]
Let $n_{\rm sel}$ be the first index at which the iterate enters the symmetric neighborhood $\mathcal V_{\rm in}$ in \eqref{eq:discrete-neighborhoods}. Then $n_{\rm sel}\le\lceil T_c/h\rceil$ and, for every $n\ge n_{\rm sel}$,
\[
 \cL_{n+1}\le(1-c_mh)\cL_n,
 \qquad
 \sup_{n\ge0}|\delta_{i,n}-\delta_{i,0}|\le C_m\frac hD
 \quad\text{for every }i,
\]
and, writing $s_n$ for the common redundant direction at iteration $n$, we have
\[
 \sum_{n\ge n_{\rm sel}}\|s_{n+1}-s_n\|
 \le\frac{C_m}{D^2}\cL_{n_{\rm sel}}.
\]
\end{theorem}
The imbalance drift is small enough to preserve the separation of learning rates. To show this, we express the original update as an exact Euler step in linearly scaled parameters and compare its interpolation with the limiting flow. Near the selected representation, a Hessian bound over each complete step gives loss decrease and prevents escape. Summing these decreases controls both direction motion and imbalance drift (\cref{thm:discrete-capture}).

\FloatBarrier

\section{Experiments}
Our analysis shows that parameter scales can select a unique specialized neuron while redundant coefficients decay. Here we test these dynamics and their dependence on imbalance, sample size, and step size, then examine feature assignment with two competing teachers. We compare the large-$D$ limit, the exact finite-$D$ equations, gradient flow in the original parameters on the population loss, and full-batch flow on $N$ Gaussian samples. Appendix~\ref{app:validation} gives the protocols and complete results.

\begin{figure}[tbp]
\centering
\includegraphics[width=0.95\linewidth]{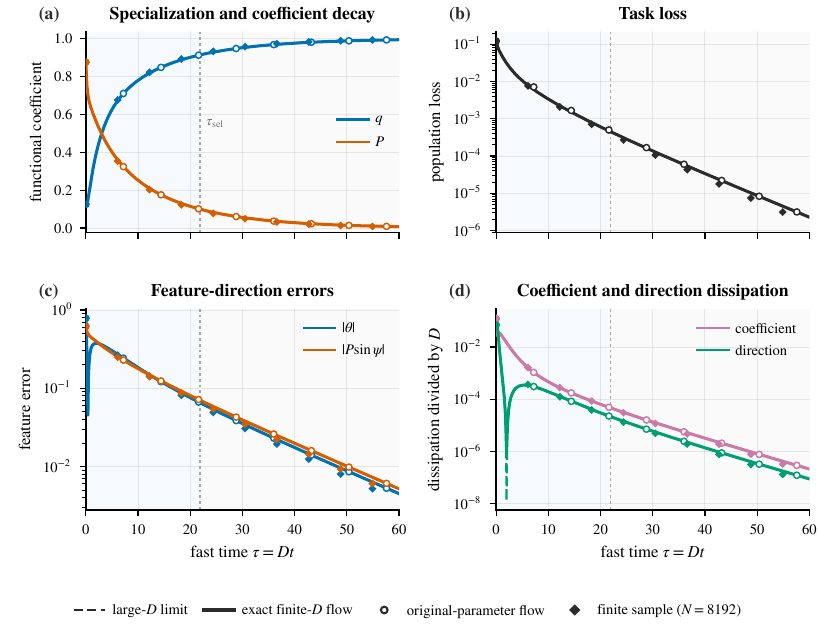}
\caption{A designated neuron learns the teacher as redundant coefficients decay. The four dynamical descriptions agree along a representative trajectory with $m=8$, $D=32$, and $N=8192$. (a) Selected coefficient $q$ and total redundant coefficient $P$. The vertical line marks $\tau_{\rm sel}=Dt_{\rm sel}$, defined by the thresholds in \eqref{eq:capture}. (b) Population loss. (c) Selected angle $|\theta|$ and coefficient-weighted redundant direction error $|P\sin\psi|$. (d) Loss decrease due to coefficient and direction changes, divided by $D$.}
\label{fig:trajectories}
\end{figure}

\subsection{Specialization and coefficient decay}
Figure~\ref{fig:same-predictor} compares two scale assignments at the same initial functional state. With all imbalances positive, the students learn the teacher together and retain equal coefficients. With one positive and seven negative imbalances, the selected student aligns with the teacher and its coefficient approaches one. Each redundant coefficient tends to zero while its direction remains separated from the teacher. Figure~\ref{fig:trajectories} follows this process in all four dynamical descriptions. The coefficient, loss, and direction panels (a--c) show consistent behavior across the four descriptions. The loss decomposition (d) shows the contributions of coefficient adjustment and feature rotation. Across the population grid, the reduced and original-parameter flows agree to solver precision, and the limiting and empirical flows follow the same pattern.

To isolate the initial difference in adaptation, \cref{fig:single-paths} compares single students with the same initial coefficient and alignment. Positive imbalance produces rapid alignment, while negative imbalance gives coefficient adjustment followed by slower rotation. Their specialization times differ by the predicted quadratic factor, and the paths themselves differ. The limiting approximations capture both phases (\cref{fig:diagnostics}(a)): slow-time alignment is compared from time zero, while the coefficient approximation applies after its initial relaxation.

The two direction curves in Figure~\ref{fig:trajectories}(c) reflect different mechanisms. The selected angle approaches zero, but $|P\sin\psi|$ decreases through coefficient decay: the redundant angle remains nearly fixed (Figure~\ref{fig:same-predictor}). Since the redundant students remain identical, each coefficient vanishes individually.

\subsection{Scaling and stability}
\begin{figure}[tbp]
\centering
\includegraphics[width=0.95\linewidth]{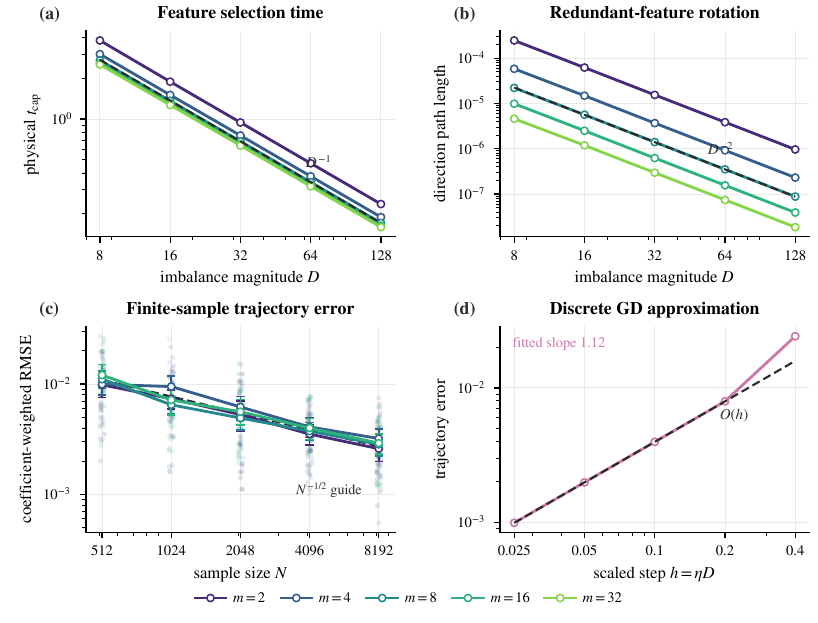}
\caption{Scaling with imbalance, sample size, and step size. (a) Selection time follows $D^{-1}$. (b) Subsequent redundant-direction motion follows $D^{-2}$. (c) Empirical trajectory error follows approximately $N^{-1/2}$, with paired-bootstrap 95\% confidence intervals. (d) Gradient-descent error is approximately linear in $h=\eta D$. Dashed lines mark reference orders.}
\label{fig:scaling}
\end{figure}
We next vary the imbalance to test how the separation of learning rates controls selection. Figures~\ref{fig:scaling}(a,b) show faster selection and smaller subsequent direction motion with the predicted $D^{-1}$ and $D^{-2}$ orders across widths. As the redundant directions become more nearly fixed, the limiting-trajectory error also decreases quadratically (\cref{fig:diagnostics}(b), \cref{tab:population}).

Measuring direction motion after selection tests whether the remaining neurons continue to change their features. The observed decrease supports the total-motion bound in \cref{thm:main-capture}, connecting rapid feature acquisition to a persistent separation of selected and redundant directions.

For the finite-sample experiment, inputs lie in the two-dimensional teacher plane. Let $s_{\rm sel}$ and $s_{\rm red}$ denote the selected and common redundant unit directions. To compare directions in proportion to their contribution to the predictor, we use the coefficient-weighted state
\begin{equation}
\Xfun(\tau)
=
\bigl(q(\tau),\,P(\tau),\,q(\tau)s_{\mathrm{sel}}(\tau),\,P(\tau)s_{\mathrm{red}}(\tau)\bigr)
\in\R^6,
\label{eq:functionalstate}
\end{equation}
Let $\Xfun_N$ and $\Xfun_{\rm pop}$ be this state along the empirical and population flows, respectively.  Their root-mean-square trajectory error is
\begin{equation}
\err_N
=
\left[
\frac1{6T}
\int_0^T
\|\Xfun_N(\tau)-\Xfun_{\mathrm{pop}}(\tau)\|_2^2\,\dd\tau
\right]^{1/2},
\label{eq:sampleerror}
\end{equation}
with $T=60$.  Weighting each direction by its coefficient prevents a nearly vanishing unit from dominating the trajectory metric.

The same selected neuron is recovered in every finite-sample run, with trajectory error decreasing approximately as $N^{-1/2}$ (\cref{fig:scaling}(c), \cref{tab:samples}). Selection persists under the tested perturbations of the initial coefficients, directions, and imbalances. Figure~\ref{fig:diagnostics}(c) confirms selection across the tested angles and widths. Panel (d) shows the weak dependence of rescaled selection time on width. Under gradient descent, the trajectory error decreases linearly with $h=\eta D$ (\cref{fig:scaling}(d)), while subsequent redundant-direction motion retains its $D^{-2}$ dependence (\cref{tab:discrete-transport}).

\Needspace{10\baselineskip}
\subsection{Competition between two teacher features}
We next train eight students with nonduplicate initial directions on two teacher features in ten dimensions. Every student's initial coefficient and direction remain fixed as the imbalance assignment changes. Configurations A and B give positive imbalance to units $(1,4)$ and $(2,3)$, respectively, with one designated student near each teacher; a third configuration uses $(1,2)$, both nearer the first teacher. Success requires the designated students to recover their teachers' directions and coefficients, with all remaining coefficients small (\eqref{eq:two-teacher-success}).

\begin{figure}[tbp]
\centering
\includegraphics[width=0.95\linewidth]{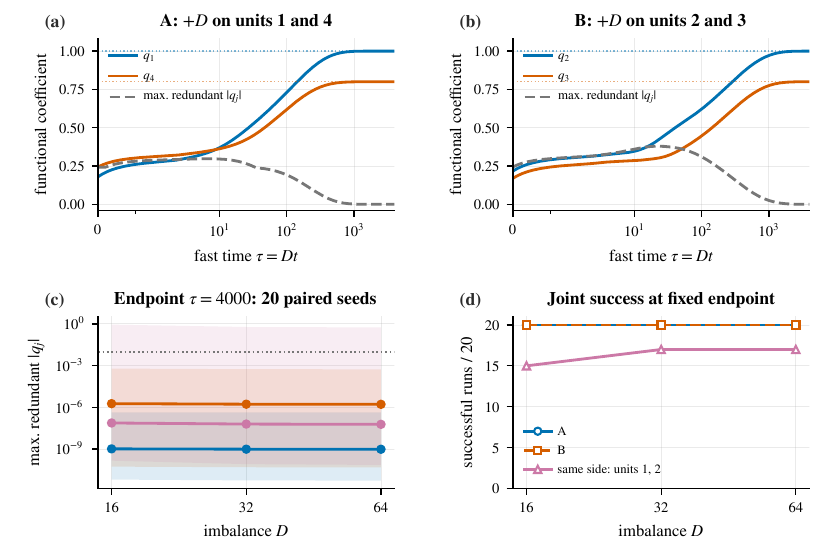}
\caption{Two-teacher competition at a fixed initial predictor.  (a)--(b) At $D=32$ and seed 0, moving the positive imbalances from units $(1,4)$ to $(2,3)$ changes which students learn the two teachers.  Gray curves show the largest magnitude among coefficients of neurons with negative imbalance.  (c) Endpoint redundant-coefficient magnitude: lines are medians and bands span the minimum to maximum over 20 paired seeds.  The dotted line is $10^{-2}$.  (d) Counts satisfying \eqref{eq:two-teacher-success}.  A and B overlap at 20/20.  The same-side condition tests competition between two designated students initially nearer the first teacher.}
\label{fig:two-teacher}
\end{figure}

Moving the positive imbalances from A to B changes which neurons acquire the two features (\cref{fig:two-teacher}(a,b)). The fixed initial features isolate the effect of scale from differences in alignment. Both configurations succeed for every seed at each tested imbalance (\cref{tab:two-teacher}). When both designated students begin nearer the first teacher, other students can retain part of the second feature. These experiments show how parameter scale and initial direction jointly determine feature assignment beyond the duplicate setting.

\FloatBarrier

\Needspace{11\baselineskip}
\section{Conclusion}
Parameter scale can determine how neurons specialize even when their initial features and functional contributions are fixed. In a Gaussian ReLU model, we prove that a designated student learns the teacher while every redundant coefficient vanishes. The result holds at finite imbalance under gradient flow and small-step population gradient descent. Two-teacher experiments further show that scale can change feature assignment among nonduplicate students, linking parameterization to the roles neurons acquire during learning.

\bibliographystyle{plainnat}
\bibliography{references}

\clearpage
\appendix
\raggedbottom
\makeatletter
\@addtoreset{equation}{section}
\@addtoreset{figure}{section}
\@addtoreset{table}{section}
\makeatother
\renewcommand{\theequation}{\thesection.\arabic{equation}}
\renewcommand{\thefigure}{\thesection.\arabic{figure}}
\renewcommand{\thetable}{\thesection.\arabic{table}}

\section*{Appendices}
Appendix~A derives the coefficient and direction dynamics and proves the single-neuron results.  Appendix~B establishes global convergence of the limiting system and the local loss bounds.  Appendix~C uses these estimates to prove selection and convergence at finite imbalance, under initialization perturbations, and for gradient descent.  Appendix~D gives the experimental protocols and additional results.

We use $O(\cdot)$ for upper bounds and $\Theta(\cdot)$ for two-sided bounds, with other parameters fixed.  Subscripts record allowed dependence of the constants.  Generic positive constants $c,C$ may change between estimates.

\section{Exact coefficient and direction dynamics and one-neuron specialization}
\label{app:marked}

\subsection{Population gradient flow in coefficient and direction variables}
The same ReLU feature can be represented by different input and output norms. We derive equations that keep this scale information while tracking the coefficient and direction seen by the predictor. The particle notation below includes the finite network as a sum of point masses.

Let $\pi_x$ be a probability measure on $\R^d$ with finite second moment and assume that it assigns zero mass to every hyperplane used below.  Let $y\in L^2(\pi_x)$ and $\relu(z)=z_+=\max\{z,0\}$.  For a positive measure $\varrho$ with finitely many particles on
\[
 \cP:=\R\times(\R^d\setminus\{0\}),
\]
define
\begin{equation}
\label{eq:model-original}
 f_\varrho(x)=\int a\,\relu(w^\top x)\,\varrho(\dd a,\dd w),
 \qquad
 \cL(\varrho)=\frac12\int\bigl(f_\varrho(x)-y(x)\bigr)^2\,\pi_x(\dd x).
\end{equation}
For the finite network in the main text, $\varrho$ is the sum of unit point masses at the neuron parameters, so the integral is exactly the network sum.  Write $R_\varrho=f_\varrho-y$.  Along a characteristic of the population gradient flow in its measure formulation,
\begin{equation}
\label{eq:original-characteristics}
 \dot a=-\int R_\varrho(x)\relu(w^\top x)\,\pi_x(\dd x),
 \qquad
 \dot w=-a\int R_\varrho(x)\1_{\{w^\top x>0\}}x\,\pi_x(\dd x).
\end{equation}
This is the standard two-layer mean-field characteristic system \citep{mei2019}.  The signed radial--directional transformation below agrees, for nonzero coefficients, with \citet[Appendix C.2, Eq.~(131)]{kunin2024} at equal unit learning rates.  We include its derivation to fix conventions and to retain the regular extension through $q=0$ when $\delta<0$.

For $w\neq0$, set
\begin{equation}
\label{eq:marked-coordinates}
 r=\|w\|,
 \qquad
 s=\frac{w}{r}\in\Sph^{d-1},
 \qquad
 q=ar,
 \qquad
 \delta=a^2-r^2.
\end{equation}
Define
\begin{equation}
\label{eq:mu-chi}
 \mu(q,\delta):=\sqrt{\delta^2+4q^2},
 \qquad
 r^2(q,\delta)=\frac{\mu(q,\delta)-\delta}{2},
 \qquad
 \chi(q,\delta):=\frac{2q}{\mu(q,\delta)-\delta}.
\end{equation}
The coordinate domain is
\begin{equation}
\label{eq:marked-space}
 \cX:=\bigl\{(q,s,\delta):s\in\Sph^{d-1},\ \mu(q,\delta)>\delta\bigr\}.
\end{equation}
This condition is equivalent to $r>0$.  The formula for $\chi$ is regular at $q=0$ whenever $\delta<0$ and equals $a/r$.  The inverse transformation is
\begin{equation}
\label{eq:inverse-chart}
 r=\sqrt{\frac{\mu-\delta}{2}},
 \qquad
 a=\frac q r,
 \qquad
 w=rs.
\end{equation}
Hence the map $(a,w)\mapsto(q,s,\delta)$ is a smooth bijection from $\cP$ onto $\cX$.

Let $\rho$ be the pushforward of $\varrho$ under this coordinate map, namely the distribution obtained by applying the map to each particle.  Positive homogeneity gives
\begin{equation}
\label{eq:model-marked}
 f_\rho(x)=\int q\,\relu(s^\top x)\,\rho(\dd q,\dd s,\dd\delta).
\end{equation}
With $R_\rho=f_\rho-y$, define the residual moments
\begin{equation}
\label{eq:residual-moments}
 B_\rho(s):=\int R_\rho(x)\relu(s^\top x)\,\pi_x(\dd x),
 \qquad
 \mathcal T_\rho(s):=\int R_\rho(x)\1_{\{s^\top x>0\}}x\,\pi_x(\dd x).
\end{equation}
By homogeneity,
\begin{equation}
\label{eq:B-sT}
 B_\rho(s)=s^\top\mathcal T_\rho(s).
\end{equation}
Let $\proj_s^\perp=I-ss^\top$.

\begin{theorem}
\label{thm:marked-reduction}
Let $(\varrho_t)_{t\in[0,T]}$ be a characteristic solution of \eqref{eq:original-characteristics} whose support remains in $\cP$.  Its pushforward $\rho_t$ under the map $(a,w)\mapsto(q,s,\delta)$ satisfies
\begin{equation}
\label{eq:marked-continuity}
 \partial_t\rho
 +\partial_q(V_q\rho)
 +\diver_{\Sph^{d-1}}(V_s\rho)=0,
 \qquad
 V_\delta=0,
\end{equation}
where $\operatorname{div}_{\Sph^{d-1}}$ is spherical divergence and the coordinate velocities are
\begin{equation}
\label{eq:marked-velocities}
 V_q=-\mu(q,\delta)B_\rho(s),
 \qquad
 V_s=-\chi(q,\delta)\proj_s^\perp\mathcal T_\rho(s),
 \qquad
 V_\delta=0.
\end{equation}
Conversely, every characteristic solution of \eqref{eq:marked-continuity} that remains in $\cX$ pulls back through \eqref{eq:inverse-chart} to a solution of \eqref{eq:original-characteristics}.  The two systems are therefore equivalent up to the first exit from the regular coordinate domain.

Moreover, the loss decay separates exactly into coefficient and directional contributions:
\begin{equation}
\label{eq:loss-dissipation}
 \frac{\dd}{\dd t}\cL(\rho_t)
 =-\int
 \left[
 \underbrace{\mu(q,\delta)B_{\rho_t}(s)^2}_{\text{coefficient change}}
 +
 \underbrace{\frac{\mu(q,\delta)+\delta}{2}
 \bigl\|\proj_s^\perp\mathcal T_{\rho_t}(s)\bigr\|^2}_{\text{direction change}}
 \right]\rho_t(\dd q,\dd s,\dd\delta).
\end{equation}
In particular, the entire $\delta$-marginal is conserved.
\end{theorem}

\begin{proof}
Since $\relu(w^\top x)=r\relu(s^\top x)$, \eqref{eq:original-characteristics} becomes
\begin{equation}
\label{eq:a-w-short}
 \dot a=-rB_\rho(s),
 \qquad
 \dot w=-a\mathcal T_\rho(s).
\end{equation}
Taking radial and tangential components gives
\begin{equation}
\label{eq:r-s-dot}
 \dot r=s^\top\dot w=-aB_\rho(s),
 \qquad
 \dot s=\frac1r\proj_s^\perp\dot w
 =-\frac ar\proj_s^\perp\mathcal T_\rho(s).
\end{equation}
Therefore
\begin{align}
 \dot q
 &=r\dot a+a\dot r
 =-(r^2+a^2)B_\rho(s)
 =-\mu(q,\delta)B_\rho(s),
 \label{eq:qdot-proof}
 \\
 \dot\delta
 &=2a\dot a-2r\dot r
 =-2arB_\rho(s)+2arB_\rho(s)
 =0.
 \label{eq:deltadot-proof}
\end{align}
The identity $a/r=\chi(q,\delta)$ yields the stated $s$-velocity.  Applying the chain rule to test functions proves the pushforward continuity equation.  The inverse calculation gives the converse.

For dissipation, the original Euclidean gradient flow gives
\begin{align}
 \frac{\dd}{\dd t}\cL
 &=-\int\left[r^2B_\rho(s)^2+a^2\|\mathcal T_\rho(s)\|^2\right]\dd\rho
 \notag\\
 &=-\int\left[(r^2+a^2)B_\rho(s)^2
 +a^2\|\proj_s^\perp\mathcal T_\rho(s)\|^2\right]\dd\rho,
 \label{eq:dissipation-expand}
\end{align}
where \eqref{eq:B-sT} was used to decompose $\mathcal T=B s+\proj_s^\perp\mathcal T$.  Using
\[
 r^2+a^2=\mu(q,\delta),
 \qquad
 a^2=\frac{\mu(q,\delta)+\delta}{2},
\]
proves \eqref{eq:loss-dissipation}.
\end{proof}

For fixed $q\ne0$, the relative directional weight
$a^2/(a^2+r^2)=\tfrac12(1+\delta/\sqrt{\delta^2+4q^2})$
increases from $0$ to $1$ as $\delta$ ranges from $-\infty$ to $+\infty$.  The coefficient rate depends on $|\delta|$, while the directional rate also depends on its sign.

\begin{remark}
\label{rem:weight-decay-marked}
Consider the regularized objective
\[
 \cL_{\lambda_{\rm wd}}(a,w)
 =\cL(a,w)+\frac{\lambda_{\rm wd}}{2}\bigl(a^2+\|w\|^2\bigr),
 \qquad \lambda_{\rm wd}\ge0.
\]
The original-parameter flow becomes
\[
 \dot a=-rB-\lambda_{\rm wd}a,
 \qquad
 \dot w=-a\mathcal T-\lambda_{\rm wd}w.
\]
On every regular neuron, the variables $(q,s,\delta)$ satisfy the exact identities
\begin{equation}
\label{eq:weight-decay-marked}
 \dot q=-\mu(q,\delta)B-2\lambda_{\rm wd}q,
 \qquad
 \dot s=-\chi(q,\delta)\proj_s^\perp\mathcal T,
 \qquad
 \dot\delta=-2\lambda_{\rm wd}\delta.
\end{equation}
Hence $\delta(t)=\delta(0)e^{-2\lambda_{\rm wd}t}$.  Symmetric weight decay preserves its sign and reduces its magnitude.  On the $O(D^{-1})$ time scale for $\delta=+D$, fixed $\lambda_{\rm wd}$ changes $\delta$ by a relative $O(\lambda_{\rm wd}/D)$ amount.  Retaining the unregularized $O(D)$ specialization scale for $\delta=-D$ requires $\lambda_{\rm wd}D$ to be small.
\end{remark}

\subsection{Effect of node rescaling beyond two layers}

The following proposition gives the effect of node rescaling on the gradient at a fixed parameter state.

\begin{proposition}
\label{prop:deep-node-mobility}
Consider a finite feedforward ReLU computation graph with independently parameterized edges and no normalization layers.  Fix a hidden node with preactivation
\[
 z=u^\top h+b,
\]
where $h$ is the vector of activations entering the node, with outgoing parameter vector $v$ and all remaining parameters collected in $\xi$.  Write the incoming affine parameter as $\bar u:=(u,b)$.  The bias-free case omits $b$ from the parameter block.  Assume $\bar u\neq0$ and define, for $\rho>0$,
\[
 \bar u^{(\rho)}:=\bar u/\rho,
 \qquad
 v^{(\rho)}:=\rho v,
 \qquad
 \Theta^{(\rho)}:=g_\rho\Theta:=(\bar u^{(\rho)},v^{(\rho)},\xi).
\]
Then the realized network function is invariant:
\begin{equation}
\label{eq:deep-node-function-invariance}
 f_{g_\rho\Theta}=f_\Theta.
\end{equation}
Let $\cL$ be any differentiable loss that depends on the parameters only through the realized network outputs.  At every differentiability point,
\begin{equation}
\label{eq:deep-node-gradient-scaling}
 \nabla_{\bar u^{(\rho)}}\cL(\Theta^{(\rho)})
 =\rho\,\nabla_{\bar u}\cL(\Theta),
 \qquad
 \nabla_{v^{(\rho)}}\cL(\Theta^{(\rho)})
 =\rho^{-1}\nabla_v\cL(\Theta),
 \qquad
 \nabla_\xi\cL(\Theta^{(\rho)})=\nabla_\xi\cL(\Theta).
\end{equation}
Under Euclidean gradient flow, write $\widehat{\bar u}=\bar u/\|\bar u\|$ and, when $v\neq0$, $\widehat v=v/\|v\|$.  The instantaneous normalized-direction velocities at the two rescaled states satisfy
\begin{equation}
\label{eq:deep-node-direction-mobility}
 \left.\frac{\dd}{\dd t}\widehat{\bar u^{(\rho)}}\right|_{\Theta^{(\rho)}}
 =\rho^2\left.\frac{\dd}{\dd t}\widehat{\bar u}\right|_\Theta,
 \qquad
 \left.\frac{\dd}{\dd t}\widehat{v^{(\rho)}}\right|_{\Theta^{(\rho)}}
 =\rho^{-2}\left.\frac{\dd}{\dd t}\widehat v\right|_\Theta.
\end{equation}
Node rescaling changes the relative speeds of the incoming and outgoing directions.  If both velocities are nonzero and $\rho\neq1$, their joint velocity is not proportional to the original velocity.
\end{proposition}

\begin{proof}
Positive homogeneity gives
\[
 \relu\!\left((u/\rho)^\top h+b/\rho\right)
 =\rho^{-1}\relu(u^\top h+b).
\]
Multiplying every outgoing edge by $\rho$ leaves each downstream preactivation unchanged.  The complete network output is therefore unchanged, which proves \eqref{eq:deep-node-function-invariance}.

The identity
\[
 \cL(\bar u/\rho,\rho v,\xi)=\cL(\bar u,v,\xi)
\]
holds for all $(\bar u,v,\xi)$ in every differentiability region.  Differentiating it with respect to the three parameter blocks yields \eqref{eq:deep-node-gradient-scaling}.  Under Euclidean gradient flow,
\[
 \dot{\bar u}^{(\rho)}
 =-\nabla_{\bar u^{(\rho)}}\cL(\Theta^{(\rho)})
 =\rho\dot{\bar u},
 \qquad
 \dot v^{(\rho)}
 =-\nabla_{v^{(\rho)}}\cL(\Theta^{(\rho)})
 =\rho^{-1}\dot v.
\]
For any nonzero vector $z$, writing $\widehat z=z/\|z\|$,
\[
 \frac{\dd}{\dd t}\frac z{\|z\|}
 =\frac1{\|z\|}\left(I-\widehat z\widehat z^\top\right)\dot z.
\]
The normalized directions are unchanged by positive scaling, while $\|\bar u^{(\rho)}\|=\|\bar u\|/\rho$ and $\|v^{(\rho)}\|=\rho\|v\|$.  Equation~\eqref{eq:deep-node-direction-mobility} follows.  The two nonzero blocks are multiplied by distinct factors $\rho^2$ and $\rho^{-2}$, which proves the final claim.
\end{proof}

\begin{corollary}
\label{cor:node-mobility-marked-link}
Consider a two-layer scalar-output ReLU neuron with fixed $q=ar>0$.  Start from the balanced representation $a=r=\sqrt q$ and apply the reciprocal rescalings $\rho$ and $\rho^{-1}$ with $\rho\ge1$.  The resulting imbalance variables are
\[
 \delta_+=q(\rho^2-\rho^{-2}),
 \qquad
 \delta_-=-q(\rho^2-\rho^{-2}),
\]
and their input-direction speeds are $\rho^2$ and $\rho^{-2}$ times the balanced speed.  Their ratio is therefore $\rho^4$.  If $D=|\delta_\pm|$ and $q$ is fixed, then $\rho^2=\Theta(D)$ and $\rho^4=\Theta(D^2)$ as $D\to\infty$.
\end{corollary}

\begin{remark}
\label{rem:local-global-mobility}
\Cref{prop:deep-node-mobility} applies at any differentiability point for any data distribution, loss, and network depth.  The global results use additional control along the trajectory and the Gaussian teacher--student structure.  They are proved for bias-free two-layer networks.
\end{remark}

\begin{remark}
\label{rem:gauge-equivariant-updates}
Let $\mathcal U$ be a deterministic one-step update map and suppose it is exactly equivariant under node rescaling:
\[
 \mathcal U(g_\rho\Theta)=g_\rho\mathcal U(\Theta)
 \qquad\text{for every admissible }(\rho,\Theta).
\]
If $\widetilde\Theta_0=g_\rho\Theta_0$, induction gives $\widetilde\Theta_n=g_\rho\Theta_n$ for every $n$.  Hence $f_{\widetilde\Theta_n}=f_{\Theta_n}$ throughout training.  The continuous-time analogue is a vector field satisfying $V(g_\rho\Theta)=Dg_\rho(\Theta)V(\Theta)$, where $Dg_\rho$ denotes the derivative of the rescaling map.  An update with this equivariance preserves the common function trajectory of rescaled initializations.  Path-SGD is motivated by a rescaling-invariant path geometry \citep{neyshabur2015}.  
\end{remark}

\subsection{Gaussian ReLU kernel and closed ODE}

Let $d\ge2$ and $x\sim\mathcal N(0,I_d)$.  Fix a teacher
\begin{equation}
\label{eq:teacher-student}
 f_\star(x)=q_\star\relu(s_\star^\top x),
 \qquad
 q_\star>0,
 \qquad
 \|s_\star\|=1,
\end{equation}
and a student
\[
 f(x)=q\relu(s^\top x),
 \qquad q>0,
 \qquad \|s\|=1.
\]
Set $c=s^\top s_\star$.  Define the arc-cosine kernel
\begin{equation}
\label{eq:kappa}
 \kappa(c)
 :=\mathbb E[\relu(s^\top x)\relu(s_\star^\top x)]
 =\frac{\sqrt{1-c^2}+(\pi-\arccos c)c}{2\pi}.
\end{equation}
It satisfies
\begin{equation}
\label{eq:kappa-properties}
 \kappa(-1)=0,
 \qquad
 \kappa(1)=\frac12,
 \qquad
 \kappa'(c)=\frac{\pi-\arccos c}{2\pi}>0
 \quad(-1<c\le1).
\end{equation}
Closed-form population gradients for Gaussian ReLU teacher--student models are classical \citep{tian2017}.  Here we combine them with the conserved scale-imbalance variable $\delta$.

\begin{lemma}
\label{lem:gaussian-moments}
For the one-teacher model,
\begin{equation}
\label{eq:B-teacher}
 B(s)=\frac q2-q_\star\kappa(c),
\end{equation}
and
\begin{equation}
\label{eq:PT-teacher}
 \proj_s^\perp\mathcal T(s)
 =-q_\star\kappa'(c)\proj_s^\perp s_\star.
\end{equation}
The population loss is
\begin{equation}
\label{eq:loss-teacher}
 \cL(q,c)=\frac{q^2}{4}+\frac{q_\star^2}{4}-qq_\star\kappa(c).
\end{equation}
\end{lemma}

\begin{proof}
The identities $\mathbb E[\relu(s^\top x)^2]=1/2$ and \eqref{eq:kappa} give \eqref{eq:B-teacher} and \eqref{eq:loss-teacher}.  For a tangent perturbation $v\perp s$,
\[
 v^\top\mathbb E\left[x\1_{\{s^\top x>0\}}\relu(s_\star^\top x)\right]
 =\frac{\dd}{\dd\epsilon}\bigg|_{\epsilon=0}
 \kappa\!\left(\frac{s+\epsilon v}{\|s+\epsilon v\|}\cdot s_\star\right)
 =\kappa'(c)v^\top s_\star.
\]
The self term is parallel to $s$, so tangential projection yields \eqref{eq:PT-teacher}.
\end{proof}

\begin{theorem}
\label{thm:teacher-ode}
Fix a scale imbalance $\delta\in\R$.  On the regular region $q>0$ and $-1<c<1$, population gradient flow is exactly
\begin{equation}
\label{eq:teacher-ode}
 \dot q
 =-\mu(q,\delta)\left(\frac q2-q_\star\kappa(c)\right),
 \qquad
 \dot c
 =\chi(q,\delta)q_\star\kappa'(c)(1-c^2),
 \qquad
 \dot\delta=0.
\end{equation}
The loss dissipates according to
\begin{equation}
\label{eq:teacher-loss-diss}
 \dot\cL
 =-\mu(q,\delta)
 \left(\frac q2-q_\star\kappa(c)\right)^2
 -\frac{\mu(q,\delta)+\delta}{2}
 q_\star^2\kappa'(c)^2(1-c^2).
\end{equation}
\end{theorem}

\begin{proof}
Insert \eqref{eq:B-teacher}--\eqref{eq:PT-teacher} into \cref{thm:marked-reduction} and use
\[
 \|\proj_s^\perp s_\star\|^2=1-c^2.
\]
\end{proof}

\subsection{Invariant region and global convergence}
To compare specialization times, we first need to know that each trajectory reaches the teacher. The coefficient bounds below keep the alignment equation regular and allow us to follow the solution for all time.

Define
\begin{equation}
\label{eq:h-def}
 h(c):=2q_\star\kappa(c).
\end{equation}
Then the coefficient equation is
\begin{equation}
\label{eq:q-tracking}
 \dot q=-\frac{\mu(q,\delta)}{2}\bigl(q-h(c)\bigr).
\end{equation}

\begin{theorem}
\label{thm:global-specialization}
Assume
\[
 q(0)=q_0>0,
 \qquad
 -1<c(0)=c_0<1.
\]
Let
\begin{equation}
\label{eq:q-bounds}
 \underline q:=\min\{q_0,h(c_0)\}>0,
 \qquad
 \overline q:=\max\{q_0,q_\star\}.
\end{equation}
Then the solution of \eqref{eq:teacher-ode} exists for all $t\ge0$ and satisfies
\begin{equation}
\label{eq:invariant-rectangle}
 \underline q\le q(t)\le\overline q,
 \qquad
 c_0\le c(t)<1.
\end{equation}
Moreover,
\begin{equation}
\label{eq:global-limit}
 c(t)\nearrow1,
 \qquad
 q(t)\longrightarrow q_\star,
 \qquad
 \cL(q(t),c(t))\longrightarrow0.
\end{equation}
\end{theorem}

\begin{proof}
Since $h$ is increasing and $0<h(c)\le q_\star$ on $c\in[c_0,1]$, the vector field at $q=\underline q$ points inward:
\[
 q=\underline q\Longrightarrow q-h(c)\le\underline q-h(c_0)\le0\Longrightarrow\dot q\ge0.
\]
At $q=\overline q$,
\[
 q=\overline q\Longrightarrow q-h(c)\ge\overline q-q_\star\ge0\Longrightarrow\dot q\le0.
\]
Thus $q\in[\underline q,\overline q]$.  On this compact interval $\chi(q,\delta)>0$.  Since $\kappa'(c)>0$ on $(-1,1)$, $c$ is strictly increasing until it reaches $1$, and $c=1$ is invariant.

Let $c_\infty=\lim_t c(t)$.  If $c_\infty<1$, then on the compact rectangle $[\underline q,\overline q]\times[c_0,c_\infty]$ the coefficient
\[
 \chi(q,\delta)q_\star\kappa'(c)(1-c^2)
\]
has a strictly positive lower bound, contradicting convergence of $c(t)$.  Hence $c_\infty=1$.

For any $\eta>0$, eventually $|h(c(t))-q_\star|<\eta$.  Outside $[q_\star-2\eta,q_\star+2\eta]$, equation \eqref{eq:q-tracking} then gives a drift of magnitude at least $\eta\min_{[\underline q,\overline q]}\mu/2$ toward this interval.  The trajectory enters it in finite time and cannot subsequently leave.  Letting $\eta\downarrow0$ proves $q(t)\to q_\star$, and \eqref{eq:loss-teacher} gives $\cL\to0$.
\end{proof}

\subsection{Function-preserving rescaling and quadratic timescale separation}

Fix $D>0$ and a common coefficient and direction $(q,s)$.  The two scale imbalances $\delta=\pm D$ are realized by
\begin{equation}
\label{eq:two-gauges}
 \begin{aligned}
 a_\pm^2&=\frac{\sqrt{D^2+4q^2}\pm D}{2},
 &
 r_\pm^2&=\frac{\sqrt{D^2+4q^2}\mp D}{2},
 &
 a_\pm r_\pm&=q.
 \end{aligned}
\end{equation}
They have identical predictors and losses.  At the same state $(q,c)$, their coefficient rate factor is identical:
\begin{equation}
\label{eq:reaction-same}
 \mu(q,+D)=\mu(q,-D)=\sqrt{D^2+4q^2}.
\end{equation}
Their directional rate factors are
\begin{equation}
\label{eq:chi-pm}
 \chi_+(q,D)=\frac{\sqrt{D^2+4q^2}+D}{2q},
 \qquad
 \chi_-(q,D)=\frac{\sqrt{D^2+4q^2}-D}{2q},
 \qquad
 \chi_+\chi_-=1.
\end{equation}

For $0<\varepsilon<1-c_0$, define the specialization hitting time
\begin{equation}
\label{eq:hitting-time}
 T_\delta(\varepsilon)
 :=\inf\{t\ge0:c_\delta(t)\ge1-\varepsilon\}.
\end{equation}
Let
\begin{equation}
\label{eq:Ieps}
 I_{c_0}(\varepsilon)
 :=\int_{c_0}^{1-\varepsilon}\frac{\dd c}{1-c^2}
 =\frac12\log\!\left(
 \frac{(2-\varepsilon)(1-c_0)}{\varepsilon(1+c_0)}
 \right),
\end{equation}
and $k_0:=\kappa'(c_0)>0$.

\begin{theorem}
\label{thm:D2-gap}
Under the assumptions of \cref{thm:global-specialization}, let $D\ge\overline q$.  For the two trajectories with the same initial $(q_0,c_0)$ and scale imbalances $\delta=\pm D$,
\begin{align}
 \frac{\underline q}{Dq_\star}I_{c_0}(\varepsilon)
 &\le T_{+D}(\varepsilon)
 \le\frac{\overline q}{Dq_\star k_0}I_{c_0}(\varepsilon),
 \label{eq:Tplus-bounds}
 \\
 \frac{2D}{\overline q q_\star}I_{c_0}(\varepsilon)
 &\le T_{-D}(\varepsilon)
 \le\frac{2D}{\underline q q_\star k_0}I_{c_0}(\varepsilon).
 \label{eq:Tminus-bounds}
\end{align}
It follows that
\begin{equation}
\label{eq:ratio-bounds}
 \frac{2k_0}{\overline q^{2}}D^2
 \le
 \frac{T_{-D}(\varepsilon)}{T_{+D}(\varepsilon)}
 \le
 \frac{2}{\underline q^{2}k_0}D^2.
\end{equation}
In particular, uniformly for $\varepsilon\downarrow0$,
\begin{equation}
 T_{+D}(\varepsilon)=\Theta\!\left(D^{-1}\log\frac1\varepsilon\right),
 \qquad
 T_{-D}(\varepsilon)=\Theta\!\left(D\log\frac1\varepsilon\right).
\end{equation}
\end{theorem}

\begin{proof}
Because $c$ is strictly increasing, we may write $q_\delta(c)$ for the coefficient evaluated when the alignment equals $c$.  Thus
\begin{equation}
\label{eq:hitting-integral}
 T_\delta(\varepsilon)
 =\int_{c_0}^{1-\varepsilon}
 \frac{\dd c}{\chi(q_\delta(c),\delta)
 q_\star\kappa'(c)(1-c^2)}.
\end{equation}
By \cref{thm:global-specialization}, $\underline q\le q_\delta(c)\le\overline q$.  For $D\ge\overline q$,
\begin{equation}
\label{eq:chi-plus-bounds}
 \frac{D}{\overline q}\le\chi_+(q,D)\le\frac{2D}{\underline q},
\end{equation}
where the upper bound uses
\[
 \sqrt{D^2+4q^2}\le D+\frac{2q^2}{D}.
\]
For $\delta=-D$,
\begin{equation}
\label{eq:chi-minus-bounds}
 \frac{\underline q}{2D}\le\chi_-(q,D)\le\frac{\overline q}{D},
\end{equation}
using $\chi_-=2q/(\sqrt{D^2+4q^2}+D)$.  Also,
\[
 k_0\le\kappa'(c)\le\frac12
 \qquad(c_0\le c\le1).
\]
Insert these bounds into \eqref{eq:hitting-integral} to obtain \eqref{eq:Tplus-bounds}--\eqref{eq:Tminus-bounds}. Taking ratios gives \eqref{eq:ratio-bounds}. On $0<\varepsilon\le(1-c_0)/2$, the ratio $I_{c_0}(\varepsilon)/\log(1/\varepsilon)$ is positive and continuous, and extends to $\varepsilon=0$ with value $1/2$. It therefore has positive lower and finite upper bounds, giving the uniform orders in \cref{thm:main-gap}.
\end{proof}

\begin{corollary}
\label{cor:not-time-rescaling}
Fix $D>0$, $q_0>0$, and $-1<c_0<1$.  Let $(q_\pm,c_\pm)$ solve \eqref{eq:teacher-ode} with $\delta=\pm D$ and common initial state $(q_0,c_0)$.  For any $t_1>0$ and strictly increasing $C^1$ map $\tau:[0,t_1]\to[0,\infty)$ with $\tau(0)=0$, the identity
\[
 (q_+(t),c_+(t))=(q_-(\tau(t)),c_-(\tau(t)))
 \qquad(0\le t\le t_1)
\]
cannot hold throughout $[0,t_1]$.
At a common state off the nullcline $q/2=q_\star\kappa(c)$ the two vector fields are non-collinear.  On the nullcline their phase curves have different second derivatives.
\end{corollary}

\begin{proof}
Put
\[
 R(q,c)=\frac q2-q_\star\kappa(c),
 \qquad
 A(c)=q_\star\kappa'(c)(1-c^2)>0.
\]
Because $c_\pm$ are strictly increasing while $-1<c_\pm<1$, each trajectory has a local phase representation $q=Q_\pm(c)$.  Equation \eqref{eq:teacher-ode} gives
\begin{equation}
\label{eq:phase-curve-slope}
 Q_\pm'(c)
 =-\frac{\mu(Q_\pm(c),D)}{\chi(Q_\pm(c),\pm D)A(c)}
 R(Q_\pm(c),c).
\end{equation}
A time reparameterization would give a common phase curve $Q$ on a
nontrivial $c$-interval.  Subtracting its two equations in
\eqref{eq:phase-curve-slope}, and using $\chi(q,+D)\ne\chi(q,-D)$,
forces $R(Q(c),c)=0$ throughout that interval.  The phase equation then
gives $Q'=0$, whereas $Q(c)=2q_\star\kappa(c)$ gives
$Q'=2q_\star\kappa'(c)>0$, a contradiction.

The local geometry follows from the same equation.  Off the nullcline the
slopes differ.  On it, $Q_\pm'(c_0)=0$ and differentiation gives
\[
 Q_-''(c_0)-Q_+''(c_0)
 =\frac{\mu(q_0,D)D}{q_0(1-c_0^2)}>0.
\]
\end{proof}

\subsection{Fast and slow large-imbalance limits}
The time bounds suggest different limiting descriptions for the two signs of imbalance. Positive imbalance permits alignment on the fast time, whereas negative imbalance produces an initial coefficient relaxation followed by slow alignment. We keep these two stages separate when approximating the coefficient.

The bounds above hold at finite $D$.  The following limits describe the trajectories on the two time scales.

\begin{proposition}
\label{prop:fast-plus}
Let
\[
 Q_D^+(\tau)=q_{+D}(\tau/D),
 \qquad
 C_D^+(\tau)=c_{+D}(\tau/D).
\]
For every $0<\tau_\mathrm{max}<\infty$, there exist constants
$D_0<\infty$ and $C_{\tau_\mathrm{max}}<\infty$, depending only on
$(q_0,c_0,q_\star,\tau_\mathrm{max})$, such that for every $D\ge D_0$,
\begin{equation}
\label{eq:fast-plus-conv}
 \sup_{0\le\tau\le\tau_\mathrm{max}}
 \bigl\|(Q_D^+(\tau),C_D^+(\tau))-(Q_+(\tau),C_+(\tau))\bigr\|
 \le\frac{C_{\tau_\mathrm{max}}}{D^2},
\end{equation}
where
\begin{equation}
\label{eq:fast-plus-limit}
 \begin{aligned}
  \frac{\dd Q_+}{\dd\tau}
  &=-\left(\frac{Q_+}{2}-q_\star\kappa(C_+)\right),
  &
  \frac{\dd C_+}{\dd\tau}
  &=\frac{q_\star}{Q_+}\kappa'(C_+)(1-C_+^2),\\
  Q_+(0)&=q_0,
  &
  C_+(0)&=c_0.
 \end{aligned}
\end{equation}
\end{proposition}

\begin{proof}
On the invariant rectangle $[\underline q,\overline q]\times[c_0,1]$,
\[
 \frac{\mu(q,+D)}{D}=1+O(D^{-2}),
 \qquad
 \frac{\chi(q,+D)}{D}=\frac1q+O(D^{-2}).
\]
The expansions and the Lipschitz constants of the vector fields are uniform on this rectangle.  Continuous dependence and Gronwall's inequality give \eqref{eq:fast-plus-conv}.
\end{proof}

\begin{proposition}
\label{prop:slow-minus}
Let $h(c)=2q_\star\kappa(c)$ and $e_D(t)=q_{-D}(t)-h(c_{-D}(t))$.  There is a constant $C_{\mathrm{tr}}<\infty$ that depends only on $(q_0,c_0,q_\star)$.  For every $D>0$,
\begin{equation}
\label{eq:tracking-estimate}
 |e_D(t)|
 \le |e_D(0)|e^{-Dt/2}+\frac{C_{\mathrm{tr}}}{D^2}
 \qquad(t\ge0).
\end{equation}
On the fast scale $\tau=Dt$, $(q_{-D}(\tau/D),c_{-D}(\tau/D))$ converges on compact intervals to
\begin{equation}
\label{eq:fast-minus-limit}
 \begin{aligned}
  \frac{\dd Q_f}{\dd\tau}
  &=-\left(\frac{Q_f}{2}-q_\star\kappa(c_0)\right),
  &
  \frac{\dd C_f}{\dd\tau}&=0,\\
  Q_f(0)&=q_0,
  &
  C_f(0)&=c_0.
 \end{aligned}
\end{equation}
The limiting solution is therefore
\begin{equation}
 Q_f(\tau)=h(c_0)+\bigl(q_0-h(c_0)\bigr)e^{-\tau/2},
 \qquad C_f(\tau)=c_0.
\end{equation}
On the slow scale $\sigma=t/D$, define $C_D^-(\sigma):=c_{-D}(D\sigma)$.  This trajectory converges uniformly on compact intervals to the solution of
\begin{equation}
\label{eq:slow-minus-limit}
 \frac{\dd C_-}{\dd\sigma}
 =2q_\star^2\kappa(C_-)\kappa'(C_-)(1-C_-^2),
 \qquad C_-(0)=c_0.
\end{equation}
For every $0<\sigma_0<\sigma_\mathrm{max}<\infty$, there exist constants
$D_0<\infty$ and $C_{\sigma_0,\sigma_\mathrm{max}}<\infty$.  They depend
only on $(q_0,c_0,q_\star,\sigma_0,\sigma_\mathrm{max})$, and for every
$D\ge D_0$,
\begin{equation}
 \sup_{\sigma_0\le\sigma\le\sigma_\mathrm{max}}
 \bigl|q_{-D}(D\sigma)-2q_\star\kappa(C_-(\sigma))\bigr|
 \le \frac{C_{\sigma_0,\sigma_\mathrm{max}}}{D^2}.
\end{equation}
\end{proposition}

\begin{proof}
Set $A(c)=q_\star\kappa'(c)(1-c^2)$.  Since $q\in[\underline q,\overline q]$,
\[
 |h'(c)|\le q_\star,
 \qquad
 0\le A(c)\le\frac{q_\star}{2},
 \qquad
 0\le\chi(q,-D)\le\frac{\overline q}{D}.
\]
Differentiating $e_D=q-h(c)$ gives
\[
 \dot e_D
 =-\frac{\mu(q,-D)}2e_D-h'(c)\chi(q,-D)A(c).
\]
Since $\mu\ge D$, variation of constants yields \eqref{eq:tracking-estimate}.  The fast-scale limit follows from
\[
 \frac{\mu(q,-D)}{D}=1+O(D^{-2}),
 \qquad
 \frac{\chi(q,-D)}{D}=O(D^{-2}).
\]

For the slow scale, the exact identity
\begin{equation}
\label{eq:Dchi-q}
 \left|D\chi(q,-D)-q\right|
 =\frac{4|q|^3}{\bigl(\sqrt{D^2+4q^2}+D\bigr)^2}
 \le\frac{\overline q^{3}}{D^2}
\end{equation}
shows
\[
 \frac{\dd}{\dd\sigma}C_D^-(\sigma)
 =q_{-D}(D\sigma)A(C_D^-(\sigma))+O(D^{-2}).
\]
The transient term in \eqref{eq:tracking-estimate} has slow-time integral $O(D^{-2})$, and the remaining tracking error is also $O(D^{-2})$.  Gronwall's inequality gives uniform convergence of $C_D^-$ to \eqref{eq:slow-minus-limit}.  The coefficient statement follows from \eqref{eq:tracking-estimate}.
\end{proof}

\section{Global fast dynamics and local geometry}
\label{app:global-selection}

Let $d\ge2$, $X\sim\mathcal N(0,I_d)$, and let the teacher be
\[
 f_\star(x)=\relu(u^\top x),\qquad \|u\|=1.
\]
The student network is
\[
 f(x)=\sum_{i=1}^{m}q_i(t)\relu(s_i(t)^\top x),
 \qquad \|s_i(t)\|=1.
\]
We use population square loss
\[
 \cL=\frac12\E\bigl[(f(X)-f_\star(X))^2\bigr].
\]
All students begin with the same coefficient and feature direction,
\begin{equation}
\label{eq:init-all}
 q_i(0)=\frac1m,
 \qquad
 s_i(0)=s_0,
 \qquad
 0<\phi=\angle(s_0,u)\le\frac\pi3.
\end{equation}
Choose one index $k$ and assign
\[
 \delta_k=+D,
 \qquad
 \delta_j=-D\quad(j\ne k),
\]
where $\delta_i=a_i^2-\|w_i\|^2$ is the conserved scale imbalance and $q_i=a_i\|w_i\|$.  The underlying positive factors are uniquely reconstructed from
\[
 a_i^2=\frac{\sqrt{D^2+4q_i^2}+\delta_i}{2},
 \qquad
 \|w_i\|^2=\frac{\sqrt{D^2+4q_i^2}-\delta_i}{2}.
\]
Thus every choice of the selected index gives the same initial function and loss.

Set
\[
 M:=m-1.
\]
By permutation symmetry, the $M$ students with $\delta=-D$ remain identical.  Write
\[
 q=q_k,
 \qquad
 p_{\rm red}=q_j\ (j\ne k),
 \qquad
 P:=M p_{\rm red},
\]
Let $e=(s_0-\cos\phi\,u)/\sin\phi$ and $s_\alpha=\cos\alpha\,u+\sin\alpha\,e$.  Define $\theta$ and $\psi$ as the signed angles of the direction of unit $k$
and the common redundant direction relative to $u$ in the invariant plane
$\operatorname{span}\{u,s_0\}$.  The population function is exactly
\[
 f(x)=q\relu(s_\theta^\top x)+P\relu(s_\psi^\top x).
\]

Define the angular ReLU kernel
\[
 g(\alpha)
 =\frac{\sin\alpha+(\pi-\alpha)\cos\alpha}{2\pi},
 \qquad
 \omega(\alpha):=-g'(\alpha)
 =\frac{(\pi-\alpha)\sin\alpha}{2\pi}.
\]
This is the angular form of the kernel in Appendix~A:
$g(\alpha)=\kappa(\cos\alpha)$ for $\alpha\in[0,\pi]$.
In angular formulas, $g$ is extended evenly and periodically with period $2\pi$.  The reduced loss is
\begin{equation}
\label{eq:reduced-loss}
 \mathscr E(q,P,\theta,\psi)
 =\frac14(q^2+P^2+1)
 +qP g(|\psi-\theta|)
 -qg(|\theta|)-Pg(|\psi|).
\end{equation}
The self and pair interactions of the redundant units sum to $P^2/4$ because their directions coincide.

Let
\[
 \mu_+(q)=\sqrt{D^2+4q^2},
 \qquad
 \mu_-(p_{\rm red})=\sqrt{D^2+4p_{\rm red}^2},
\]
\[
 \chi_+(q)=\frac{\mu_+(q)+D}{2q},
 \qquad
 \chi_-(p_{\rm red})=\frac{2p_{\rm red}}{\mu_-(p_{\rm red})+D}.
\]

\begin{proposition}
\label{prop:exact-reduction}
On the duplicate-redundant manifold, ordinary population gradient flow is exactly
\begin{align}
 \dot q&=-\mu_+(q)\,\partial_q\mathscr E,
 &
 \dot P&=-M\mu_-(P/M)\,\partial_P\mathscr E,
 \label{eq:exact-amplitude}\\
 \dot\theta&=-\frac{\chi_+(q)}{q}\,\partial_\theta\mathscr E,
 &
 \dot\psi&=-\frac{\chi_-(P/M)}{P}\,\partial_\psi\mathscr E,
 \label{eq:exact-angle}
\end{align}
where the last expression has its continuous extension at $P=0$.  Moreover,
\begin{align}
 \frac{\dd}{\dd t}\mathscr E
 ={}&-\mu_+(q)(\partial_q\mathscr E)^2
      -M\mu_-(P/M)(\partial_P\mathscr E)^2\nonumber\\
 &-\frac{\chi_+(q)}{q}(\partial_\theta\mathscr E)^2
      -\frac{\chi_-(P/M)}{P}(\partial_\psi\mathscr E)^2.
 \label{eq:exact-dissipation}
\end{align}
\end{proposition}

\begin{proof}
For one redundant unit, the coefficient gradient is
\[
 \partial_{q_j}\cL
 =\frac P2+qg(|\psi-\theta|)-g(|\psi|)
 =\partial_P\mathscr E.
\]
Therefore $\dot P=M\dot p_{\rm red}$ gives the second equation in \eqref{eq:exact-amplitude}.  The common angular derivative satisfies
\[
 \partial_\psi\mathscr E
 =\sum_{j\ne k}\partial_{\psi_j}\cL,
\]
and each summand is $1/M$ of the total.  Dividing the individual angular
gradient by the individual functional coefficient $p_{\rm red}=P/M$ gives the final
equation in \eqref{eq:exact-angle}.  The equations for unit $k$ follow from
the exact coordinate reduction in \cref{thm:marked-reduction}.  Pairing the four
state velocities with the four derivatives of \eqref{eq:reduced-loss} yields
\eqref{eq:exact-dissipation}.
\end{proof}

For nonzero vectors, let $\angle(z,v)$ be the angle between their normalized directions.  Define the Gaussian ReLU covariance
\[
 K(z,v)=\E_{x\sim\gamma_d}[[z^\top x]_+[v^\top x]_+]
       =\|z\|\|v\|g(\angle(z,v)).
\]
When $\|v\|=1$ and $\gamma=\angle(z,v)$,
\begin{equation}
\label{eq:cartesian-kernel-gradient}
 \nabla_zK(z,v)=\frac{(\pi-\gamma)v+\sin\gamma\,z/\|z\|}{2\pi}.
\end{equation}
The kernel is jointly $C^2$ when both vectors are nonzero.  At the angular
endpoints, $g(t)=\tfrac12\cos t+t^3/(6\pi)+O(t^5)$ and
$g(\pi-t)=t^3/(6\pi)+O(t^5)$.  Near parallel or antiparallel vectors, the nonsmooth angular contribution is of the form $\|h_\perp\|^3$, where $h_\perp$ is a transverse displacement.  This function is $C^2$ at zero.
Differentiation and homogeneity give the uniform bounds
\begin{equation}
\label{eq:kernel-hessian}
 \|\nabla_{zz}^2K(z,v)\|\le C\frac{\|v\|}{\|z\|},\qquad
 \|\nabla_{zv}^2K(z,v)\|\le C.
\end{equation}

\subsection{Scaled Cartesian dynamics}
The angular equation contains a factor that becomes singular when the selected coefficient vanishes. Cartesian coordinates retain the same feature information without dividing by that coefficient. The following linear scaling also identifies which variables move on the fast time.

Let $(a_+,w_+)$ denote the selected unit and $(a_-,w_-)$ any of the
$M$ identical redundant units.  Introduce the linear coordinates
\begin{equation}
\label{eq:scaled-cartesian}
 b=\frac{a_+}{\sqrt D},\qquad z=\sqrt D\,w_+,\qquad
 p=M\sqrt D\,a_-,\qquad v=\frac{w_-}{\sqrt D},\qquad
 \Xi=(b,z,p,v).
\end{equation}
Here $p$ is the scaled total redundant amplitude. The functional coefficients and directions are recovered by
\begin{equation}
\label{eq:scaled-visible}
 q=b\|z\|,\quad P=p\|v\|,\quad s_+=z/\|z\|,\quad s_-=v/\|v\|.
\end{equation}
These maps are smooth wherever $z,v\ne0$, including at $p=P=0$.
In these coordinates the predictor and loss contain no $D$:
\begin{equation}
\label{eq:scaled-energy}
 f_\Xi(x)=b[z^\top x]_++p[v^\top x]_+,\qquad
 \mathcal E(\Xi)=\tfrac12\|f_\Xi-f_\star\|_{L^2(\gamma_d)}^2.
\end{equation}
The raw gradient of each redundant unit is $1/M$ of the gradient obtained
by differentiating their common parameters.  The chain rule therefore gives,
in fast time $\tau=Dt$ (primes denote $\tau$ derivatives), the exact equations
\begin{equation}
\label{eq:scaled-flow}
 \Xi'=-\mathsf M_D\nabla\mathcal E(\Xi),\qquad
 \mathsf M_D=\operatorname{diag}(D^{-2},I_d,M,(MD^2)^{-1}I_d).
\end{equation}
The initial point is
\[
 \Xi_D(0)=(1,s_0/m,M/m,s_0)+O(D^{-2}),
\]
with an error uniform in $m\ge2$.  At $D=\infty$, $b=1$ and $v=s_0$
are frozen, while $z'=-\nabla_z\mathcal E$ and $p'=-M\partial_p\mathcal E$.
Writing $z=qs_\theta$ recovers \eqref{eq:weighted-fast}.
\subsection{The limiting system}

The limit of \eqref{eq:scaled-flow} has $b=1$ and $v=s_0$.
With $z=qs_\theta$ and $p=P$, define
\begin{equation}
\label{eq:fast-energy}
 E_M(q,P,\theta)
 :=\frac14(q^2+P^2+1)
 +qP g(\phi-\theta)
 -qg(|\theta|)-Pg(\phi),
\end{equation}
for $\phi-\pi\le\theta\le\phi$.  The fast system is
\begin{equation}
\label{eq:weighted-fast}
 q'=-\partial_qE_M,
 \qquad
 P'=-M\partial_PE_M,
 \qquad
 \theta'=-q^{-2}\partial_\theta E_M,
\end{equation}
with initial condition
\begin{equation}
\label{eq:fast-init}
 q(0)=q_0:=\frac1{M+1},
 \qquad
 P(0)=P_0:=\frac M{M+1},
 \qquad
 \theta(0)=\phi.
\end{equation}
The energy landscape is independent of $M$.  The subscript records the metric
weight in the flow:
\begin{equation}
\label{eq:weighted-dissipation}
 \frac{\dd}{\dd\tau}E_M
 =-(\partial_qE_M)^2
  -M(\partial_PE_M)^2
  -q^{-2}(\partial_\theta E_M)^2.
\end{equation}

\subsection{Arc-cosine kernel inequalities}

\begin{lemma}
\label{lem:uniform-arc-bounds}
For $0<\phi<\pi$,
\begin{align}
 g(x)g(\phi-x)-\omega(x)\omega(\phi-x)&\le\tfrac12g(\phi)
       &&(0\le x\le\phi),\label{eq:K1}\\
 g(x)g(\phi+x)+\omega(x)\omega(\phi+x)&\le\tfrac12g(\phi)
       &&(0\le x\le\pi-\phi).\label{eq:K3}
\end{align}
The function
\begin{equation}
\label{eq:Fphi-uniform}
 F_\phi(x)=\bigl(\tfrac14-g(\phi+x)^2\bigr)\omega(x)
 -\bigl(\tfrac12g(\phi)-g(\phi+x)g(x)\bigr)\omega(\phi+x)
\end{equation}
satisfies $F_\phi(0)=0$ and $F_\phi(x)>0$ for $0<x\le\pi-\phi$.
\end{lemma}
\begin{proof}
The difference between the right and left sides of \eqref{eq:K1} is
symmetric about $\phi/2$, vanishes at zero, and has derivative
$(\phi-2x)\sin x\sin(\phi-x)/(2\pi^2)\ge0$ on $[0,\phi/2]$.
For \eqref{eq:K3}, define
\[
 R_\phi(x)=\tfrac12g(\phi)-g(x)g(\phi+x)-\omega(x)\omega(\phi+x).
\]
Then $R_\phi(0)=0$ and
$R_\phi'(x)=(2\pi-\phi-2x)\sin x\sin(\phi+x)/(2\pi^2)\ge0$.
Next put
\begin{equation}
\label{eq:kernel-H}
 H(t)=\tfrac14-g(t)^2-\omega(t)^2,\qquad
 H(0)=0,\quad H'(t)=\frac{(\pi-t)\sin^2t}{\pi^2}>0\quad(0<t<\pi).
\end{equation}
Since $F_\phi(x)=H(\phi+x)\omega(x)-R_\phi(x)\omega(\phi+x)$,
\begin{equation}
\label{eq:Fphi-monotonicity}
 \frac{\dd}{\dd x}\frac{F_\phi(x)}{\omega(\phi+x)}
 =\frac{\phi\sin x\sin(\phi+x)}{2\pi^2}
 +H(\phi+x)\frac{\dd}{\dd x}\frac{\omega(x)}{\omega(\phi+x)}>0.
\end{equation}
Here $(\log\omega)''(t)=-(\pi-t)^{-2}-\sin^{-2}t<0$, so the ratio
$\omega(x)/\omega(\phi+x)$ is strictly increasing for $0<x<\pi-\phi$.
The value at zero and
$F_\phi(\pi-\phi)=\omega(\pi-\phi)/4>0$ complete the proof.
\end{proof}

\subsection{An invariant region for the limiting system}
Before locating the limiting state, we show that the trajectory stays in a bounded region and that the selected vector cannot reach zero. A component of that vector toward the teacher supplies the needed lower bound.

Put $\alpha=\phi-\theta$ and $Z=qg(\alpha)-g(\phi)$.
The duplicate initial point belongs to
\begin{equation}
\label{eq:sector}
 \mathcal C_\phi=\{0<q\le1,\ 0\le P\le1,\
     \phi-\pi\le\theta\le\phi,\ Z\le0\},
\end{equation}
since $Z(0)=1/(2m)-g(\phi)<0$ for $0<\phi\le\pi/3$.
In this range, $g(\phi)\ge g(\pi/3)=1/6+\sqrt3/(4\pi)>1/4$.
The selected vector acquires a positive component toward the teacher
immediately, which prevents it from reaching zero.

\begin{lemma}
\label{lem:projection}
The region $\mathcal C_\phi$ is forward invariant.  For the duplicate
initialization, let
\[
 e_\perp=\frac{u-\cos\phi\,s_0}{\sin\phi},\qquad
 y=e_\perp^\top z=q\sin(\phi-\theta).
\]
Then, for all $\tau\ge0$,
\begin{equation}
\label{eq:projection-lower}
 q(\tau)\ge y(\tau)\ge
 \frac{\phi\sin\phi}{\pi}(1-e^{-\tau/2}).
\end{equation}
On any fixed compact interval $I_\phi\subset(0,\pi/3]$, a constant
$c_I>0$ independent of $m$ satisfies
\begin{equation}
\label{eq:growing-q-margin}
 q(\tau)\ge c_I\bigl(m^{-1}+\min\{\tau,1\}\bigr).
\end{equation}
\end{lemma}

\begin{proof}
Until the first exit from \eqref{eq:sector}, the Cartesian fast equation is
\[
 z'=-z/2+\nabla_zK(z,u)-P\nabla_zK(z,s_0).
\]
Using \eqref{eq:cartesian-kernel-gradient} and projecting onto $e_\perp$ gives
\begin{equation}
\label{eq:projection-identity}
 y'=-\frac y2+\frac{(\pi-|\theta|)\sin\phi+
     (\sin|\theta|-P\sin\alpha)\sin\alpha}{2\pi}.
\end{equation}
The numerator is smallest at $P=1$.
For $0\le\theta\le\phi$, its excess over $(\pi-\phi)\sin\phi$ is
\[
 \alpha\sin\phi+\sin\alpha(\sin\theta-\sin\alpha)
 \ge\sin\alpha(\sin\phi+\sin\theta-\sin\alpha)\ge0.
\]
For $\phi-\pi\le\theta\le0$, put $t=\pi-\phi+\theta$.
Its excess over $\phi\sin\phi$ is
\[
 t\sin\phi+\sin t(\sin(\phi+t)-\sin t)
 \ge\sin t(\sin\phi+\sin(\phi+t)-\sin t)\ge0.
\]
The final signs follow, respectively, from
\[
 \sin\phi+\sin\theta-\sin(\phi-\theta)
 =4\sin(\theta/2)\cos(\phi/2)\cos((\phi-\theta)/2),
\]
and
\[
 \sin\phi+\sin(\phi+t)-\sin t
 =4\sin(\phi/2)\cos(t/2)\cos((\phi+t)/2).
\]
Thus $y'\ge-y/2+\phi\sin\phi/(2\pi)$.
Since $y(0)=0$, this proves \eqref{eq:projection-lower} on the stopped
interval and excludes $q=0$ at any positive first-exit time.
For a general point in $\mathcal C_\phi$, the same argument uses $y(0)\ge0$.

All other faces are inward-pointing.  At $q=1$ and $P=1$,
\[
 \partial_qE_M=\tfrac12+Pg(\alpha)-g(|\theta|)\ge0,\qquad
 \partial_PE_M=\tfrac12+qg(\alpha)-g(\phi)>0.
\]
At $P=0$, $P'=-MZ\ge0$.  The angular velocities at $\theta=\phi$ and
$\theta=\phi-\pi$ are $-\omega(\phi)/q<0$ and
$\omega(\pi-\phi)/q>0$.
At $Z=0$, direct differentiation gives
\[
 Z'=-\tfrac12g(\phi)-P\{g(\alpha)^2+\omega(\alpha)^2\}
       +g(|\theta|)g(\alpha)
       -\operatorname{sgn}(\theta)\omega(|\theta|)\omega(\alpha)\le0,
\]
by \eqref{eq:K1} for $\theta\ge0$ and \eqref{eq:K3} for $\theta<0$.
These signs close the first-exit argument.

Finally, $q'\ge-1$ on the region, so $q(\tau)\ge m^{-1}-\tau$.
For $\tau\le1/(2m)$ this gives $q\ge1/(2m)$.  For later times use
\eqref{eq:projection-lower} and
$1-e^{-\tau/2}\ge(1-e^{-1/2})\min\{\tau,1\}$.
Reducing $c_I$ proves \eqref{eq:growing-q-margin}.
\end{proof}

\subsection{Local loss and gradient bounds}
Near the selected representation, small loss must control the coefficients and selected direction. The Hessian gives this control through the three feature variations associated with these coordinates. We first state the elementary estimate used to turn Hessian positivity into loss and gradient bounds.

\begin{lemma}
\label{lem:transverse-hessian}
Let $L(\xi;\zeta)$ be a nonnegative $C^2$ scalar loss near $\xi=0$,
with $\zeta$ in a compact parameter set.  Suppose
$L(0;\zeta)=0$, $\nabla_\xi L(0;\zeta)=0$, and
$\nabla_\xi^2L(0;\zeta)\succeq\lambda I$ uniformly, for $\lambda>0$.
On a uniform neighborhood of zero,
\begin{equation}
\label{eq:transverse-hessian-criterion}
 \frac\lambda4\|\xi\|^2\le L(\xi;\zeta)\le C\|\xi\|^2,
 \qquad \|\nabla_\xi L(\xi;\zeta)\|^2\ge\lambda L(\xi;\zeta).
\end{equation}
\end{lemma}
\begin{proof}
Hessian continuity gives $\nabla_\xi^2L\succeq\lambda I/2$ on a
smaller ball.  Taylor's theorem gives the quadratic bounds.
Strong convexity, compared with $\xi=0$, gives
$L(\xi)\le\langle\nabla L(\xi),\xi\rangle-\lambda\|\xi\|^2/4
\le\|\nabla L(\xi)\|^2/\lambda$.
\end{proof}

Near the selected solution, fix the common redundant angle $\psi$ and
write $\xi=(q-1,\theta,P)$ and $\rho=\|\xi\|$.
The transverse Hessian at zero is the explicit Gram matrix
\begin{equation}
\label{eq:symmetric-gram}
 G_\psi=\begin{pmatrix}
 1/2&0&g(\psi)\\0&1/2&\omega(\psi)\\
 g(\psi)&\omega(\psi)&1/2
 \end{pmatrix},\qquad
 \lambda_{\min}(G_\psi)=\tfrac12-\sqrt{g(\psi)^2+\omega(\psi)^2}>0.
\end{equation}
Here $0<\psi<\pi$, as holds in the neighborhoods considered below.  Its columns correspond to the selected coefficient, the selected tangent
direction, and the redundant coefficient.  Positivity follows from the
same $H(\psi)>0$ used in the global kernel argument.

\begin{lemma}
\label{lem:tube}
There are $r_0,\sigma_0,c,C>0$, uniformly in $m$ and
$\phi$ on any compact $I_\phi\subset(0,\pi/3]$, such that on
$\mathcal U_0=\{\rho<r_0,\ |\psi-\phi|<\sigma_0\}$,
\begin{align}
 c\rho^2\le\mathscr E&\le C\rho^2,\label{eq:local-loss}\\
 (\partial_q\mathscr E)^2+(\partial_P\mathscr E)^2
       +q^{-2}(\partial_\theta\mathscr E)^2&\ge c\mathscr E,
       \label{eq:local-gradient}\\
 |\partial_\psi\mathscr E|&\le C|P|\sqrt{\mathscr E}.
       \label{eq:psi-derivative}
\end{align}
\end{lemma}

\begin{proof}
The loss is $C^2$ in the regular coordinates, vanishes at $\xi=0$,
and has Hessian $G_\psi$ there.  Its least eigenvalue is uniformly
positive for $\psi$ in a small compact neighborhood of $I_\phi$.
Apply \cref{lem:transverse-hessian}.  The upper and lower bounds on $q$
make the resulting gradient bound equivalent to
\eqref{eq:local-gradient}.
Finally,
\[
 \partial_\psi\mathscr E
 =P\langle f-f_\star,\mathbf1_{\{s_\psi^\top x>0\}}
                         (\partial_\psi s_\psi)^\top x\rangle,
\]
where the brackets denote the $L^2(\gamma_d)$ inner product.  Cauchy--Schwarz proves \eqref{eq:psi-derivative}.
\end{proof}

\subsection{Uniform fast convergence}
We now combine the invariant region with the local estimates. The kernel inequalities exclude other stationary points, and compactness gives a positive gradient-to-loss ratio away from the selected state. The local bound extends this control to the state itself.

\begin{theorem}
\label{thm:fast-global}
For every $m\ge2$ and $0<\phi\le\pi/3$, the duplicate-initialized
fast flow exists globally and converges to $(1,0,0)$.
On each compact $I_\phi\subset(0,\pi/3]$, there are $c_I,C_I>0$,
independent of $m$, such that for $\tau\ge1$,
\begin{equation}
\label{eq:uniform-fast-decay}
 E_M(\tau)\le E_M(1)e^{-c_I(\tau-1)},\qquad
 \rho(\tau)\le C_I e^{-c_I(\tau-1)/2}.
\end{equation}
For $0<r<r_0$, the flow therefore reaches $\rho\le r$ by fast time
\begin{equation}
\label{eq:uniform-capture-time}
 T(r)=1+\frac2{c_I}\log\frac{C_I}{r},
\end{equation}
after enlarging $C_I$ to at least one.
\end{theorem}

\begin{proof}
By \cref{lem:projection}, all trajectories for $\tau\ge1$ lie in
the joint compact set
\[
 \mathcal K_I=\{(\phi,q,P,\theta):\phi\in I_\phi,
     \ (q,P,\theta)\in\mathcal C_\phi,\ q\ge\underline q_I\},
 \quad \underline q_I=\min_{\phi\in I_\phi}
       \frac{\phi\sin\phi}{\pi}(1-e^{-1/2})>0.
\]
Write
\[
 S=P\omega(\phi-\theta)+\operatorname{sgn}(\theta)\omega(|\theta|),
 \qquad \Gamma=(\partial_qE_M)^2+(\partial_PE_M)^2+S^2.
\]
Since $M\ge1$, \eqref{eq:weighted-dissipation} gives $E_M'\le-\Gamma$.
The only zero of $\Gamma$ is $(1,0,0)$:
for $\theta>0$, $S>0$.  At $\theta=0$, $S=0$ gives $P=0$ and then
$\partial_qE_M=0$ gives $q=1$.
For $\theta=-x<0$, coefficient stationarity requires
\[
 P=\frac{\tfrac12g(\phi)-g(\phi+x)g(x)}
         {\tfrac14-g(\phi+x)^2}.
\]
The strict sign $F_\phi(x)>0$ from \cref{lem:uniform-arc-bounds}
implies $P\omega(\phi+x)<\omega(x)$, so $S\ne0$, including at
$x=\pi-\phi$.

Near the target, \cref{lem:tube} gives $\Gamma\ge cE_M$ and
$E_M\ge c\rho^2$.  Outside this neighborhood, $\Gamma$ has a
positive minimum on $\mathcal K_I$ and $E_M$ is bounded above.
Also, a zero-loss state is stationary, so the target is the only zero
of $E_M$.  Its positive minimum outside the neighborhood and the
upper bound on $\rho$ give the second inequality below.
Combining the local and exterior bounds yields, throughout $\mathcal K_I$,
\begin{equation}
\label{eq:global-dissipation-ratio}
 \Gamma\ge c_I E_M,\qquad E_M\ge c_I\rho^2.
\end{equation}
Integrating $E_M'\le-c_I E_M$ and using the uniformly bounded initial
loss proves \eqref{eq:uniform-fast-decay} and \eqref{eq:uniform-capture-time}.
\end{proof}

\subsection{Continuation through a loss barrier}
The local estimates apply only while the trajectory remains in their neighborhood. The following lemma makes that continuation explicit: loss decrease prevents exit in the controlled coordinates, and an integrated velocity bound prevents exit in the remaining coordinates.

\begin{lemma}
\label{lem:loss-barrier}
Let a trajectory have local coordinates $(\xi,\zeta)$ in
$\mathcal U=\{\|\xi\|<r,\ \zeta\in\operatorname{int}\mathcal S\}$,
where $\mathcal S$ is compact.  Assume the vector field is regular on
the closure and that, for constants $c,\lambda,A>0$, while the trajectory remains there,
\[
 L\ge c\|\xi\|^2,\qquad -\dot L\ge\lambda DL,\qquad
 \|\dot\zeta\|\le A L/D.
\]
Put $\ell_\partial=\inf_{\|\xi\|=r,\,\zeta\in\mathcal S}L$ and
$d_0=\operatorname{dist}(\zeta(t_0),\partial\mathcal S)>0$.
If the initial state is in $\mathcal U$ and
\begin{equation}
\label{eq:common-continuation-margin}
 L(t_0)<\ell_\partial,\qquad
 \frac{A L(t_0)}{\lambda D^2}<d_0,
\end{equation}
then the trajectory stays in $\mathcal U$ for all $t\ge t_0$, with
$L(t)\le L(t_0)e^{-\lambda D(t-t_0)}$, $\xi(t)\to0$, and
$\int_{t_0}^\infty\|\dot\zeta\|\,\dd t\le A L(t_0)/(\lambda D^2)$.
\end{lemma}
\begin{proof}
Up to the first exit time, the loss is at most $L(t_0)$ and the total
$\zeta$ motion is at most $A L(t_0)/(\lambda D^2)$.
The two strict inequalities exclude the transverse and remaining
boundary faces, respectively.  Regularity excludes any other finite
obstruction.  Thus the first exit is infinite.  Integration and coercivity give
the asserted estimates.
\end{proof}

\section{Finite imbalance, perturbations, and gradient descent}
\label{app:robust-discrete}

Let $d\ge2$ and fix a compact angle interval $I_\phi\subset(0,\pi/3]$.
We retain the Gaussian teacher $f_\star=[u^\top x]_+$ and designate unit
$1$.  All constants below may depend on $I_\phi$; generic positive constants $c,C$ may change between estimates.  Unqualified function norms and inner products refer to $L^2(\gamma_d)$, and $X\sim\gamma_d$ denotes a Gaussian input.  Sums indexed by $j$ run over $j=2,\ldots,m$ unless stated otherwise.
Appendix~B shows that the limiting trajectory reaches the selected representation. Here we control how a finite imbalance, a perturbation of the initial state, or a discrete update changes that trajectory. The redundant amplitudes enter the selected unit's equation through their sum. Separating this sum from the differences between amplitudes avoids treating all interactions as independent sources of error.

\subsection{A common stability estimate}

Allow $\delta_1>0$ and $\delta_j<0$ for $j\ge2$, and measure initial
departure from the duplicate state by
\begin{equation}
\label{eq:perturbnorm}
 \Delta_0=\max_i\bigl(|q_i(0)-m^{-1}|+\|s_i(0)-s_0\|\bigr)
       +\max_i\left|\frac{|\delta_i(0)|}{D}-1\right|.
\end{equation}
Use the unaggregated linear coordinates
\begin{equation}
\label{eq:full-scaled-coordinates}
 b=a_1/\sqrt D,\quad z=\sqrt D\,w_1,\qquad
 p_j=\sqrt D\,a_j,\quad v_j=w_j/\sqrt D\quad(j\ge2).
\end{equation}
They represent $f=b[z^\top x]_++\sum_{j\ge2}p_j[v_j^\top x]_+$.
With $\mathcal E=\tfrac12\|f-f_\star\|_{L^2(\gamma_d)}^2$ and $\tau=Dt$,
the exact flow is
\begin{equation}
\label{eq:full-scaled-flow}
 b'=-D^{-2}\partial_b\mathcal E,\quad z'=-\nabla_z\mathcal E,\qquad
 p_j'=-\partial_{p_j}\mathcal E,\quad
 v_j'=-D^{-2}\nabla_{v_j}\mathcal E.
\end{equation}
Imbalance perturbations enter only through the initial point of this system.

Put $M=m-1$, $K_0(z)=K(z,s_0)$, and $K_\star(z)=K(z,u)$.
Freezing $b=1$ and $v_j=s_0$ gives the reference fast flow, whose
redundant amplitudes equal $P_\infty/M$.  Write its solution as
$x_\infty=(z_\infty,P_\infty/M,\ldots,P_\infty/M)$, where
$z_\infty=q_\infty s_{\theta_\infty}$ and
$(q_\infty,P_\infty,\theta_\infty)$ solves \eqref{eq:weighted-fast}.  Its total amplitude couples
to the selected vector, while differences between redundant amplitudes
remain constant.  The following estimate separates these two effects.

Choose $r_\star<\min_{I_\phi}\phi/8$ small enough for \cref{lem:tube}.
Fix $0<r_{\rm cap}<\min\{r_\star/4,1/4\}$ so that the loss at
$\rho\le2r_{\rm cap}$ is below the boundary loss at $\rho=r_\star/2$,
and put $\mathcal V_{\rm cap}=\{\rho<r_{\rm cap}\}$.
Using \cref{thm:fast-global}, set
\begin{equation}
\label{eq:Tstar}
 T_\star=T(r_{\rm cap}),\qquad T_{\max}=T_\star+2.
\end{equation}
Let $\tau_{\rm cap}^\infty$ be the reference trajectory's first entry time into
$\overline{\mathcal V_{\rm cap}}$.  Then $\tau_{\rm cap}^\infty\le T_\star$.  The projection lower bound also gives
\begin{equation}
\label{eq:integrated-q-margin}
 \int_0^{T_{\max}}\frac{\dd\tau}{q_\infty(\tau)}
 \le C_I(1+\log m).
\end{equation}

\begin{lemma}
\label{lem:aggregate-stability}
Let $x=(z,p_2,\ldots,p_m)$ be an absolutely continuous solution of
\begin{equation}
\label{eq:forced-active-system}
 z'=-z/2-P\nabla K_0(z)+\nabla K_\star(z)+r_z,\qquad
 p_j'=-(P/2+K_0(z)-g(\phi))+r_j,\quad P=\sum_{j\ge2}p_j.
\end{equation}
Here $r_z$ and $r_j$ are additive forcing terms.  Set $e_z=\|z-z_\infty\|$, $e_P=P-P_\infty$, and
$c_j=p_j-P/M$.  Suppose
\[
 \|r_z\|+\max_j|r_j|\le C_0\epsilon,\qquad
 e_z(0)+|e_P(0)|/M+\max_j|c_j(0)|\le C_0\epsilon.
\]
Up to any stopping time $T\le T_{\max}$ on which
$e_z<q_\infty/2$, there are $C>0$ and $\beta>1/2$, independent of $m$,
such that
\begin{align}
 \sup_{\tau\le T}e_z(\tau)&\le Cm^\beta\epsilon,
 &\sup_{\tau\le T}\max_j|c_j(\tau)|&\le C\epsilon,
 \label{eq:split-stability}\\
 |e_P(\tau)|&\le CM\epsilon e^{-M\tau/2}+Cm^\beta\epsilon,
 &\sup_{\tau\le T}\|x-x_\infty\|&\le Cm^\beta\epsilon.
 \label{eq:aggregate-stability}
\end{align}
The constants depend only on $I_\phi$ and $C_0$.
\end{lemma}

\begin{proof}
Let $\bar r=M^{-1}\sum_jr_j$.  Summing and subtracting the amplitude
equations gives the exact scalar and zero-sum equations
\begin{equation}
\label{eq:mean-zero-sum}
 e_P'=-\tfrac M2e_P-M(K_0(z)-K_0(z_\infty))+M\bar r,
 \qquad c_j'=r_j-\bar r.
\end{equation}
Gaussian Cauchy--Schwarz gives $\|\nabla K_0\|\le1/2$:
for any unit $a$, both
$\E[(a^\top X)^2\mathbf1_{\{z^\top X>0\}}]$ and
$\E[[s_0^\top X]_+^2]$ equal $1/2$.
Variation of constants therefore yields
\begin{equation}
\label{eq:aggregate-convolution}
 |e_P(\tau)|\le e^{-M\tau/2}|e_P(0)|
 +\tfrac M2\int_0^\tau e^{-M(\tau-s)/2}e_z(s)\,\dd s
 +2C_0\epsilon.
\end{equation}
To estimate the selected vector, split its interaction error as
$P\nabla K_0(z)-P_\infty\nabla K_0(z_\infty)
=P_\infty(\nabla K_0(z)-\nabla K_0(z_\infty))+e_P\nabla K_0(z)$.
The first term is nonnegative when paired with $z-z_\infty$, since
$K_0$ is convex and $P_\infty\ge0$.  The segment joining $z$ to
$z_\infty$ stays at norm at least $q_\infty/2$ under the stopping
condition.  The Hessian bound \eqref{eq:kernel-hessian} controls the
teacher term on this segment.  Together with the scalar equation,
these observations give, almost everywhere,
\[
 \begin{aligned}
 e_z'&\le-\tfrac12e_z+\frac{C}{q_\infty}e_z
                              +\tfrac12|e_P|+C_0\epsilon,\\
 (|e_P|)'&\le-\tfrac M2|e_P|+\tfrac M2e_z+MC_0\epsilon.
 \end{aligned}
\]
The weight $1/M$ cancels the coupling and decay terms directly:
\begin{equation}
\label{eq:weighted-error}
 \mathcal V=e_z+\frac{|e_P|}{M},\qquad
 \mathcal V'\le\frac{C}{q_\infty}e_z+2C_0\epsilon
              \le\frac{C}{q_\infty}\mathcal V+2C_0\epsilon.
\end{equation}
Since $\mathcal V(0)\le C_0\epsilon$, Gronwall and
\eqref{eq:integrated-q-margin} yield
\begin{equation}
\label{eq:weighted-error-bound}
 \sup_{\tau\le T}\mathcal V(\tau)\le Cm^\beta\epsilon.
\end{equation}
This bounds $e_z$.  The convolution kernel in
\eqref{eq:aggregate-convolution} has mass at most one, giving the
stated $e_P$ bound.
Direct integration of $c_j'=r_j-\bar r$ bounds the zero-sum part.
Finally, the orthogonal decomposition
\begin{equation}
\label{eq:active-error-split}
 \|x-x_\infty\|^2=e_z^2+e_P^2/M+\sum_j c_j^2
\end{equation}
proves the last assertion after increasing $\beta$ to exceed $1/2$.
Only the reference total amplitude was
required to be nonnegative.
\end{proof}

\begin{proposition}
\label{prop:mdep}
There are $A,a,C>0$ and $\nu>3/2$, depending only on $I_\phi$, such
that, if
\begin{equation}
\label{eq:D0}
 D\ge D_0(m):=Am^{\nu/2},\qquad
 \Delta_0\le\varepsilon_m:=am^{-\nu},
\end{equation}
the exact flow, with active state $x_D=(z,p_2,\ldots,p_m)$, satisfies
\begin{align}
 \sup_{\tau\le T_{\max}}\|x_D-x_\infty\|
 &\le Cm^{\nu-1}(\Delta_0+D^{-2}),\label{eq:full-stability}\\
 \sup_{\tau\le T_{\max}}(|b-1|+\max_j\|v_j-s_0\|)
 &\le C(\Delta_0+D^{-2}).\label{eq:slow-stability}
\end{align}
At the reference capture time $\tau_{\rm cap}^\infty$, the redundant
output error obeys
\begin{equation}
\label{eq:capture-output-error}
 \left\|\sum_jp_j[v_j^\top\cdot]_+
              -P_\infty[s_0^\top\cdot]_+\right\|_{L^2}
 \le Cm^{\nu-1}(\Delta_0+D^{-2}).
\end{equation}
The exponent is not optimized.
\end{proposition}

\begin{proof}
Put $\epsilon=\Delta_0+D^{-2}$ and use the initial reconstruction
$a_i^2=(\sqrt{\delta_i^2+4q_i^2}+\delta_i)/2$,
$\|w_i\|^2=(\sqrt{\delta_i^2+4q_i^2}-\delta_i)/2$.
It gives the initial assumptions of \cref{lem:aggregate-stability} and
slow-coordinate error $O(\epsilon)$.  Choose $a\le1/4$ to keep all initial
coefficients positive.
Stop when $e_z=q_\infty/2$ or at a boundary of
\begin{equation}
\label{eq:regular-stopping-region}
 \|z\|<2,\quad\sum_j|p_j|<2,\quad
 1/2<b<3/2,\quad1/2<\|v_j\|<3/2.
\end{equation}
The bounded residual gives $|b'|\le CD^{-2}$ and
$\|v_j'\|\le CD^{-2}|p_j|$, proving \eqref{eq:slow-stability} up
to the stop.  The kernel bounds then put the active equations in
the form \eqref{eq:forced-active-system}, with
\begin{equation}
\label{eq:component-defect}
 \|r_z\|+\max_j|r_j|\le C\epsilon.
\end{equation}
Apply \cref{lem:aggregate-stability} and set $\nu=\beta+1$.
Small $a$ and large $A$ make $e_z<c_I/(4m)<q_\infty/2$.
Moreover,
\[
 \sum_j|p_j|\le P_\infty+|e_P|+M\max_j|c_j|
                    \le1+C(M+m^\beta)\epsilon<2.
\]
The upper bound on $z$ and the slow-coordinate bounds exclude all
remaining faces, closing the stop on $[0,T_{\max}]$.

The reference needs at least fast time $1/2$ to reach
$q_\infty\ge3/4$, because $q_\infty(0)\le1/2$ and
$q_\infty'\le1/2$ in the invariant region.
Thus $\tau_{\rm cap}^\infty\ge1/2$, and
$M e^{-M\tau_{\rm cap}^\infty/2}\le C$.
Equation \eqref{eq:aggregate-stability} gives
$|e_P(\tau_{\rm cap}^\infty)|\le Cm^\beta\epsilon$.
The ReLU Lipschitz bound and $\sum|p_j|<2$ give
\[
 \left\|\sum_jp_j[v_j^\top\cdot]_+-P_\infty[s_0^\top\cdot]_+\right\|_{L^2}
 \le C(|e_P|+\epsilon),
\]
which proves \eqref{eq:capture-output-error}.
\end{proof}

\subsection{Specialization and coefficient decay}
The comparison estimate brings the finite-imbalance trajectory into the neighborhood constructed in Appendix~B. The remaining task is to show that it stays there. The loss barrier and redundant-direction motion bound then give convergence of the individual coefficients.

\begin{theorem}
\label{thm:global-selection}
Under exact duplicate initialization, $D\ge D_0(m)$ guarantees entry
into the symmetric neighborhood $\mathcal V_m=\{\rho<r_\star/2,\ |\psi-\phi|<r_\star/2\}$ by time
$t_{\rm cap}\le T_\star/D$.  Thereafter,
\[
 \cL(t)\le\cL(t_{\rm cap})e^{-cD(t-t_{\rm cap})},\qquad
 q_1(t)\to1,\quad s_1(t)\to u,\quad q_j(t)\to0\ (j\ge2),
\]
and each redundant direction satisfies
\[
 \int_{t_{\rm cap}}^\infty\|\dot s_j(t)\|\,\dd t
 \le\frac{C\cL(t_{\rm cap})}{(m-1)D^2}.
\]
The constants are uniform in $m$ and $\phi\in I_\phi$.
Permuting the unique positive imbalance permutes the selected neuron without
changing the initial predictor.
\end{theorem}

\begin{proof}
At limiting capture, $q_\infty$ is near one and
$P_\infty^2+\theta_\infty^2+(q_\infty-1)^2\le r_{\rm cap}^2$.
The selected coefficient and direction are Lipschitz functions of the scaled coordinates there.
Equations \eqref{eq:aggregate-stability} and \eqref{eq:slow-stability}
control the total functional coefficient and redundant angle at capture.
Thus \cref{prop:mdep}, with adjusted $a,A$, gives entry into the inner
symmetric neighborhood with a strict loss gap to the transverse boundary of $\mathcal V_m$.
Exact dissipation and \cref{lem:tube} give
$-\dot\cL\ge cD\cL$ on the outer neighborhood.  Moreover,
\[
 |\dot\psi|\le\frac{|\partial_\psi\mathscr E|}{(m-1)D}
             \le\frac{C\cL}{(m-1)D}.
\]
Increase $A$ so that the resulting total motion is smaller than the
directional margin.  Apply \cref{lem:loss-barrier} with
$\xi=(q-1,\theta,P)$ and $\zeta=\psi$.
Regularity follows from $q>1/2$ and the conserved imbalance signs,
which keep all input norms nonzero.
This proves invariance, contraction, and $P\to0$.
Symmetry gives $q_j=P/(m-1)\to0$, and permutation equivariance proves
the last assertion.
\end{proof}

\begin{theorem}
\label{thm:robust-selection}
Under \eqref{eq:D0}, define
\[
 \mathcal W_{\rm sel}^{(i)}=
 \{q_i\ge1-r_\star,\ \angle(s_i,u)\le r_\star,
       \ \|\sum_{j\ne i}q_j[s_j^\top\cdot]_+\|_{L^2}\le r_\star\}.
\]
The perturbed flow first enters $\mathcal W_{\rm sel}^{(1)}$ at a
time $t_{\rm sel}\le T_\star/D$.  No other unit enters its own selection
region before that time, and
$\angle(s_j(t_{\rm sel}),u)>r_\star$ for every $j\ge2$.
\end{theorem}

\begin{proof}
At limiting capture, \eqref{eq:full-stability} controls the selected
coefficient and direction errors, and \eqref{eq:capture-output-error} controls the
redundant output directly in Gaussian $L^2$.
The choices of $a,A$ in \cref{prop:mdep} make these errors smaller
than the fixed selection margins.  This proves entry by $T_\star/D$.
Throughout this interval, \eqref{eq:slow-stability} keeps every
redundant direction close to $s_0$, at teacher angle greater than
$r_\star$.  Such a unit cannot enter its own selection region.
\end{proof}

\subsection{Local identifiability and asymptotic convergence}

With nonduplicate students, a small sum of redundant features need not make each coefficient small. To identify the coefficients, we impose a quantitative independence condition on the local feature variations. Their Gram matrix gives both the condition and the resulting loss bound.
On the selected branch $q_1>0$, set
\[
 \mathbf h=q_1s_1-u,\qquad \beta=(q_2,\ldots,q_m),\qquad
 \mathbf s=(s_2,\ldots,s_m),\qquad \beta_j=q_j\quad(j\ge2).
\]
For $\|\mathbf h\|<1/2$, $q_1=\|u+\mathbf h\|>1/2$ and the selected contribution
equals $[(u+\mathbf h)^\top x]_+$.  With $\mathbf s$ fixed, the transverse loss is
\begin{equation}
\label{eq:local-cartesian-loss}
 L_{\mathbf s}(\mathbf h,\beta)=\tfrac12
 \left\|[(u+\mathbf h)^\top\cdot]_+-[u^\top\cdot]_+
                  +\sum_{j\ge2}\beta_j[s_j^\top\cdot]_+\right\|_{L^2(\gamma_d)}^2.
\end{equation}
Its Hessian at $(\mathbf h,\beta)=0$ is
\begin{equation}
\label{eq:full-gram}
 \mathcal G_{\mathbf s}=
 \begin{pmatrix}\tfrac12 I_d&T_{\mathbf s}\\
                 T_{\mathbf s}^\top&K_{\mathbf s}\end{pmatrix},
 \qquad (K_{\mathbf s})_{ij}=K(s_i,s_j),
\end{equation}
Here $K_{\mathbf s}$ is indexed by the redundant units $i,j\in\{2,\ldots,m\}$, and the column of $T_{\mathbf s}$ corresponding to unit $j$ is
\[
 t_j=\E[X\mathbf1_{\{u^\top X>0\}}[s_j^\top X]_+]
     =\frac{\sin\vartheta_j\,u+(\pi-\vartheta_j)s_j}{2\pi},
 \qquad \vartheta_j=\angle(u,s_j).
\]
We require a uniform bound
\begin{equation}
\label{eq:full-gram-assumption}
 \mathcal G_{\mathbf s}\succeq\lambda_0 I,\qquad\lambda_0>0,
\end{equation}
on a compact neighborhood $\mathcal S$ of redundant directions.
Positive definiteness is equivalent to the Schur-complement condition
$K_{\mathbf s}-2T_{\mathbf s}^\top T_{\mathbf s}\succ0$.
Thus redundant features remain distinguishable after removing the
part explained by first-order changes of the selected unit.
The redundant coefficient map is then injective, so nonzero coefficients
cannot cancel as functions.

\begin{lemma}
\label{lem:hyperplane-gram}
If the unoriented hyperplanes $u^\perp,s_2^\perp,\ldots,s_m^\perp$
are pairwise distinct, then $\mathcal G_{\mathbf s}\succ0$.
For fixed $m,d$, a uniform positive separation between these hyperplanes
gives a uniform positive lower eigenvalue on compact direction sets.
\end{lemma}
\begin{proof}
A vector in the Gram nullspace would give
\[
 \mathbf1_{\{u^\top x>0\}}\mathbf h^\top x
       +\sum_{j\ge2}\beta_j[s_j^\top x]_+=0
\]
almost everywhere.  Gaussian full support makes this identity valid
on every open cell cut out by the hyperplanes.  At a generic point of
$u^\perp$, away from the other hyperplanes, its jump is $\mathbf h^\top x$.
Thus $\mathbf h=\alpha u$.  The identity is now a continuous piecewise linear
combination of ReLU features.  Across a generic point of $s_j^\perp$,
the gradient jump is $\beta_js_j$, so $\beta_j=0$.  Across $u^\perp$
it is $\alpha u$, so $\alpha=0$.  Continuity of the Gram matrix and
compactness prove the uniform assertion.
\end{proof}

For unit directions, the separation in this lemma is measured by
$\min\{\angle(s_i,s_j),\pi-\angle(s_i,s_j)\}$, also including $u$.
Its quantitative size matters: two redundant directions at angle
$\vartheta$ give, by the coefficient difference vector,
\begin{equation}
\label{eq:near-duplicate-gram}
 \lambda_{\min}(\mathcal G_{\mathbf s})
 \le\tfrac12-g(\vartheta)=\tfrac14\vartheta^2+O(\vartheta^3).
\end{equation}
Consequently the local radius, entry loss, and direction margin below
must still be matched to $\lambda_0$ near duplicate directions.
Under exact duplication, symmetry first combines the redundant
coefficients into one effective feature.  The resulting Gram matrix,
restricted to the teacher plane, is precisely \eqref{eq:symmetric-gram}.

\begin{lemma}
\label{lem:robust-pl}
Assume \eqref{eq:full-gram-assumption} uniformly on $\mathcal S$.
For some $0<r<1/2$, uniformly on $\|(\mathbf h,\beta)\|\le r$ and
$\mathbf s\in\mathcal S$,
\begin{equation}
\label{eq:full-gram-coercivity}
 L_{\mathbf s}\ge\frac{\lambda_0}{4}(\|\mathbf h\|^2+\|\beta\|^2),\qquad
 \|\nabla_{\mathbf h,\beta}L_{\mathbf s}\|^2\ge\lambda_0 L_{\mathbf s}.
\end{equation}
If the conserved imbalances satisfy $\delta_1\ge D/2$ and
$\delta_j\le-D/2$ for $j\ge2$, the population flow also obeys
\begin{equation}
\label{eq:robust-pl}
 -\dot\cL\ge\tfrac12\lambda_0 D\cL,\qquad
 \|\dot{\mathbf s}\|_2\le\frac{C}{\sqrt{\lambda_0}D}\cL
\end{equation}
while it remains in this neighborhood.
\end{lemma}

\begin{proof}
The scalar loss \eqref{eq:local-cartesian-loss} is $C^2$ by the regularity
of $K$ on nonzero inputs.  Its value and gradient vanish at the origin,
and its Hessian there is \eqref{eq:full-gram}.
Apply \cref{lem:transverse-hessian} to obtain
\eqref{eq:full-gram-coercivity} uniformly on $\mathcal S$.

For the selected vector $q_1s_1$, the coefficient--direction equations
give the mobility matrix
\[
 \mathsf A_1=\mu_1s_1s_1^\top+a_1^2(I-s_1s_1^\top)\succeq(D/2)I.
\]
Each redundant coefficient has mobility $\mu_j\ge D/2$.
The remaining directional contributions to dissipation are nonnegative,
so $-\dot\cL\ge(D/2)\|\nabla_{\mathbf h,\beta}L_{\mathbf s}\|^2$.
This proves the first bound in \eqref{eq:robust-pl}.
For a redundant unit,
\[
 |\chi_j|\le2|q_j|/D,\qquad
 \|\mathcal T(s_j)\|\le\sqrt{2\cL}.
\]
Taking the Euclidean product norm and using
$\sum_{j\ge2}q_j^2\le4\cL/\lambda_0$ proves the second bound.
\end{proof}

\begin{theorem}
\label{thm:robust-locking}
Suppose a trajectory with $\delta_1\ge D/2$ and $\delta_j\le-D/2$
enters a neighborhood of the selected solution at time $t_0$, with $q_1(t_0)>0$.
Let $\mathcal S$ be a compact neighborhood of $\mathbf s(t_0)$ satisfying
\eqref{eq:full-gram-assumption}, with
$d_0=\operatorname{dist}(\mathbf s(t_0),\partial\mathcal S)>0$.
Interiors and boundaries are relative to the product of direction spheres.
Choose $r<1/2$ as in \cref{lem:robust-pl} and assume that the entry state
lies in $\mathcal U=\{\|(\mathbf h,\beta)\|<r,\ \mathbf s\in\operatorname{int}\mathcal S\}$.
Define
\begin{equation}
\label{eq:transverse-loss-barrier}
 \partial_\perp\mathcal U=\{\|(\mathbf h,\beta)\|=r,\ \mathbf s\in\mathcal S\},
 \qquad \ell_\partial=\inf_{\partial_\perp\mathcal U}\cL
                    \ge\lambda_0r^2/4.
\end{equation}
For a universal constant $C>0$, if
\begin{equation}
\label{eq:continuation-margin}
 \cL(t_0)<\ell_\partial,\qquad
 \frac{C\cL(t_0)}{\lambda_0^{3/2}D^2}\le d_0/2,
\end{equation}
then the trajectory stays in $\mathcal U$ and
\[
 \cL(t)\le\cL(t_0)e^{-\lambda_0D(t-t_0)/2},\qquad
 q_1(t)\to1,\quad s_1(t)\to u,\quad q_j(t)\to0\ (j\ge2).
\]
In addition,
\begin{equation}
\label{eq:redundant-motion}
 \int_{t_0}^\infty\|\dot{\mathbf s}(t)\|_2\,\dd t
 \le\frac{C\cL(t_0)}{\lambda_0^{3/2}D^2}.
\end{equation}
\end{theorem}

\begin{proof}
Apply \cref{lem:loss-barrier} with $\xi=(\mathbf h,\beta)$ and
$\zeta=\mathbf s$, using \cref{lem:robust-pl}.
The loss gap excludes transverse exit and the stopped motion bound
excludes directional exit.  The coordinate field is regular on the
closure: $\|u+\mathbf h\|>1/2$, and the negative imbalance keeps every redundant
input norm positive, including when its coefficient is zero.
Coercivity then gives $(\mathbf h,\beta)\to0$, which is precisely the stated
specialization and coefficient-wise pruning.
\end{proof}

\subsection{Discrete full-batch gradient descent}
For gradient descent, the imbalance is no longer conserved. We use a linear scaling of the original parameters so that each discrete update remains an exact Euler step. This permits a trajectory comparison before selection and a direct bound on the accumulated imbalance drift afterward.

Consider population gradient descent in the original parameters,
\begin{equation}
\label{eq:gd}
 \Theta_{n+1}=\Theta_n-\eta\nabla_\Theta\cL(\Theta_n),\qquad \eta=h/D.
\end{equation}
Write $\cL_n=\cL(\Theta_n)$ and $\delta_{i,n}=a_{i,n}^2-\|w_{i,n}\|^2$.  Duplicate symmetry is preserved at every step.  Because the coordinate
change \eqref{eq:scaled-cartesian} is linear, this update is exactly
\begin{equation}
\label{eq:scaled-euler}
 \Xi_{n+1}=\Xi_n-h\mathsf M_D\nabla\mathcal E(\Xi_n).
\end{equation}
Thus the same Cartesian vector field governs both gradient flow and
gradient descent.
For comparison with the fast limit, write
\begin{equation}
\label{eq:pre-capture-state}
 Y_n=(q_n,P_n,\theta_n,\psi_n),\qquad
 Y_\infty(\tau)=(q_\infty(\tau),P_\infty(\tau),\theta_\infty(\tau),\phi).
\end{equation}

\begin{lemma}
\label{lem:discrete-shadow}
Let $T_c=T_\star+1$ and let $\nu>3/2$ be as in \cref{prop:mdep}.
There are $A,a>0$, depending only on $I_\phi$, such that
\begin{equation}
\label{eq:polynomial-step}
 D\ge Am^{\nu/2},\qquad 0<h\le a m^{-(\nu+1)}
\end{equation}
imply
\begin{equation}
\label{eq:discrete-shadow}
 \max_{nh\le T_c}\|Y_n-Y_\infty(nh)\|\le C_m(h+D^{-2}).
\end{equation}
All these steps, including their full Euler segments, stay in a compact
regular Cartesian neighborhood.  On these steps,
\begin{equation}
\label{eq:pre-descent}
 \cL_{n+1}\le\cL_n-\frac\eta2\|\nabla_\Theta\cL(\Theta_n)\|^2.
\end{equation}
\end{lemma}

\begin{proof}
Interpolate the exact update \eqref{eq:scaled-euler} linearly and
unaggregate its redundant amplitudes.  On each interval $[nh,(n+1)h]$,
the interpolant has the constant derivative of the finite-$D$ field
at its left endpoint.  Stop at the boundaries used in
\cref{prop:mdep}, including $e_z=q_\infty/2$.
Up to the stop, the residual is bounded, so each active component
moves by at most $Ch$ per step.  The slow derivatives still satisfy
$|b'|\le CD^{-2}$ and $\|v_j'\|\le CD^{-2}|p_{j,n}|$.
Their total changes are $O(D^{-2})$ on $[0,T_{\max}]$.
Also, the total amplitude moves by at most $Cmh$ within a step,
and $\|z\|\ge c_I/(2m)$.  The kernel Hessian bound
\eqref{eq:kernel-hessian} therefore gives the componentwise defect
\begin{equation}
\label{eq:euler-component-defect}
 \|r_z\|+\max_j|r_j|\le C(D^{-2}+mh)
\end{equation}
from \eqref{eq:forced-active-system}, almost everywhere along the
interpolant.  The same bound applies to a partial stopped segment:
its endpoints and joining lines remain in the enlarged regular neighborhood
when $mh$ is small compared with the $1/m$ radial margin.

Apply \cref{lem:aggregate-stability} with
$\epsilon=D^{-2}+mh$.  Condition \eqref{eq:polynomial-step} makes
$\epsilon\le c m^{-\nu}$ for a sufficiently small $c$, so the
stopping argument in \cref{prop:mdep} closes for complete Euler
segments as well.  In particular, $h\le1$ ensures that segments
starting before $T_c$ end before $T_{\max}$.
The change to coefficient and direction coordinates is Lipschitz on the
resulting regular neighborhood for fixed $m$, proving \eqref{eq:discrete-shadow}.
The aggregate estimate also gives the direct capture-output bound
\eqref{eq:capture-output-error}, with $\Delta_0+D^{-2}$ replaced
by $D^{-2}+mh$.

In the collapsed coordinates, the only growing kernel Hessian factor
on this neighborhood is $1/\|z\|$, so $\|\nabla^2\mathcal E\|\le C_I m$.
Put $Q_n=\langle\nabla\mathcal E(\Xi_n),\mathsf M_D\nabla\mathcal E(\Xi_n)\rangle$, where the brackets are the Euclidean inner product in scaled parameter space.
Since $\|\mathsf M_D\|\le m$ for $D\ge1$,
$\|\mathsf M_D\nabla\mathcal E\|^2\le mQ_n$.
Taylor's theorem on the complete segment gives a decrease of at least
$hQ_n/2$ when $h\le c_I/m^2$, which follows from
\eqref{eq:polynomial-step} after reducing $a$.
The chain rule, with redundant multiplicity included, gives
\begin{equation}
\label{eq:scaled-raw-dissipation}
 DQ_n=\|\nabla_\Theta\cL(\Theta_n)\|^2,
\end{equation}
and hence \eqref{eq:pre-descent}.
\end{proof}

\subsubsection{Local descent and direction motion}
Choose nested neighborhoods of the selected zero-loss manifold,
\begin{equation}
\label{eq:discrete-neighborhoods}
 \mathcal V_{\rm in}\Subset\mathcal V_{\rm mid}\Subset\mathcal V_{\rm out}.
\end{equation}
Here $\Subset$ means that the closure of each set lies in the interior of the next.  In scaled coordinates these are product neighborhoods in the transverse
variables $(q-1,P,\theta)$ and the remaining variables
$(b,\|v\|,\psi)$, centered at $(0,0,0)$ and $(1,1,\phi)$, respectively.
All three have $b,\|z\|,\|v\|$ bounded away from zero.
Choose the transverse radii $r_{\rm in}<r_{\rm mid}<r_{\rm out}$ so that
\begin{equation}
\label{eq:discrete-loss-gap}
 \sup_{\overline{\mathcal V_{\rm in}}}\mathcal E
 \le C r_{\rm in}^2<c r_{\rm mid}^2
 \le\inf_{\partial_\perp\mathcal V_{\rm mid}}\mathcal E.
\end{equation}
Here $\partial_\perp$ denotes the boundary in the transverse variables, and the last infimum allows all remaining coordinates in the closure of
their middle neighborhood.  The limiting path enters strictly inside
$\mathcal V_{\rm in}$ by time $T_\star$ after choosing the smaller entry
set in \cref{prop:mdep}.  All radii remain independent of $m,D$.

\begin{lemma}
\label{lem:scaled-local}
On $\mathcal V_{\rm out}$, uniformly for $m\ge2$ and $D\ge1$,
\begin{equation}
\label{eq:scaled-local-bounds}
 \|\nabla^2\mathcal E\|\le K_I,\qquad
 Q:=\langle\nabla\mathcal E,\mathsf M_D\nabla\mathcal E\rangle
 \ge c_I\mathcal E.
\end{equation}
Furthermore,
\begin{equation}
\label{eq:scaled-slow-bounds}
 |\partial_b\mathcal E|\le C_I\sqrt{\mathcal E},\qquad
 \|\nabla_v\mathcal E\|\le C_I|p|\sqrt{\mathcal E}
 \le C_I\mathcal E.
\end{equation}
\end{lemma}

\begin{proof}
The collapsed loss and the chosen neighborhood do not depend on $m$.
The Gaussian Cartesian kernel has bounded second derivatives when its
arguments are nonzero and bounded, giving the Hessian estimate.
For the gradient bound, $q=b\|z\|$ and $P=p\|v\|$ give
\[
 \|\nabla_z\mathcal E\|^2+M(\partial_p\mathcal E)^2
 =b^2(\partial_q\mathscr E)^2
  +\|z\|^{-2}(\partial_\theta\mathscr E)^2
  +M\|v\|^2(\partial_P\mathscr E)^2\ge c_I\mathcal E
\]
by \cref{lem:tube}.  These terms are contained in $Q$.
If $R=f_\Xi-f_\star$, then
\[
 \partial_b\mathcal E=\langle R,[z^\top x]_+\rangle,\qquad
 \nabla_v\mathcal E=p\,\E[R\mathbf1_{\{v^\top X>0\}}X].
\]
Cauchy--Schwarz and $|p|\le C_I\sqrt{\mathcal E}$, from local coercivity
and the lower bound on $\|v\|$, prove \eqref{eq:scaled-slow-bounds}.
\end{proof}

\begin{lemma}
\label{lem:discrete-continuation}
Suppose an iterate enters $\mathcal V_{\rm in}$ at index $n_{\rm cap}$
and the preceding steps satisfy \eqref{eq:pre-descent}.
For $D\ge A_I$ and $0<h\le a_I/m$, with fixed $A_I,a_I>0$,
all later iterates stay in
$\mathcal V_{\rm mid}$, and
\begin{equation}
\label{eq:discrete-contract}
 \cL_{n+1}\le(1-c_Ih)\cL_n\qquad(n\ge n_{\rm cap}).
\end{equation}
The common redundant direction $s_{-,n}=v_n/\|v_n\|$ satisfies
\begin{equation}
\label{eq:discrete-redundant-path}
 \sum_{n\ge n_{\rm cap}}\|s_{-,n+1}-s_{-,n}\|
 \le\frac{C_I}{D^2}\cL_{n_{\rm cap}}.
\end{equation}
For every original neuron,
\begin{equation}
\label{eq:gauge-drift}
 \sum_{n\ge0}|\delta_{i,n+1}-\delta_{i,n}|
 \le2\eta\cL_0=\frac{2h}{D}\cL_0.
\end{equation}
\end{lemma}

\begin{proof}
Work up to the first exit of a complete Euler segment from
$\mathcal V_{\rm mid}$.  An iterate in its closure has a fixed positive
distance from the complement of $\mathcal V_{\rm out}$.
The vector field has norm at most $C_I m$ there, so $h\le c_I/m$
places its entire segment in $\mathcal V_{\rm out}$.
For $0\le t\le1$, Taylor's theorem and \cref{lem:scaled-local} give
\begin{align}
 \mathcal E(\Xi_n-th\mathsf M_D\nabla\mathcal E(\Xi_n))
 &\le\cL_n-thQ_n+\tfrac12K_I t^2h^2mQ_n\nonumber\\
 &\le\cL_n-\tfrac12thQ_n
 \qquad(hK_Im\le1).
 \label{eq:segment-descent}
\end{align}
At $t=1$ this proves \eqref{eq:discrete-contract}, with a smaller $c_I$.
At every intermediate $t$ the loss is at most $\cL_n$, so the strict
gap \eqref{eq:discrete-loss-gap} excludes a transverse first exit.

The other coordinates have small total motion directly from
\eqref{eq:scaled-euler} and \eqref{eq:scaled-slow-bounds}:
\begin{equation}
\label{eq:direct-slow-step}
 |b_{n+1}-b_n|\le\frac{C_Ih}{D^2}\sqrt{\cL_n},\qquad
 \|v_{n+1}-v_n\|\le\frac{C_Ih}{MD^2}\cL_n.
\end{equation}
Summing the geometric loss bound, up to the stopped segment and including
any fraction of its last step, yields
\[
 \sum|b_{n+1}-b_n|\le C_ID^{-2}\sqrt{\cL_{n_{\rm cap}}},\qquad
 \sum\|v_{n+1}-v_n\|\le\frac{C_I}{MD^2}\cL_{n_{\rm cap}}.
\]
For large $D$, these are smaller than the fixed margins from the inner
neighborhood to the $b$, radial, and directional boundaries of the middle
neighborhood.  Thus no first exit is possible.  Normalization of $v$ is
Lipschitz on this regular set, proving \eqref{eq:discrete-redundant-path}.

Finally, homogeneity gives
$a_i\partial_{a_i}\cL=w_i^\top\nabla_{w_i}\cL$ and hence the exact identity
\begin{equation}
\label{eq:gauge-step}
 \delta_{i,n+1}-\delta_{i,n}
 =\eta^2\bigl[(\partial_{a_i}\cL_n)^2-
                    \|\nabla_{w_i}\cL_n\|^2\bigr].
\end{equation}
The pre-selection descent and \eqref{eq:segment-descent}, together with
\eqref{eq:scaled-raw-dissipation}, give
$\cL_{n+1}\le\cL_n-\eta\|\nabla_\Theta\cL_n\|^2/2$ at every step.
Taking absolute values in \eqref{eq:gauge-step} and telescoping the loss
therefore proves \eqref{eq:gauge-drift}.
\end{proof}

\begin{theorem}
\label{thm:discrete-capture}
For $m\ge2$ and $\phi\in I_\phi$, duplicate-initialized gradient
descent with $\eta=h/D$ satisfying \eqref{eq:polynomial-step}
(with adjusted $A,a$)
obeys \eqref{eq:discrete-shadow} and enters $\mathcal V_{\rm in}$ within
$\lceil T_c/h\rceil$ iterations.  All subsequent iterates remain in
$\mathcal V_{\rm mid}$ and satisfy \eqref{eq:discrete-contract},
\eqref{eq:discrete-redundant-path}, and \eqref{eq:gauge-drift}.
Moreover, $q_n\to1$, $P_n\to0$, and $\theta_n\to0$:
the designated unit specializes and every redundant coefficient vanishes.
\end{theorem}

\begin{proof}
At reference entry, the interpolation estimates in
\cref{lem:discrete-shadow} bound the selected and total-amplitude
errors by $Cm^{\nu-1}(D^{-2}+mh)$.  Rounding this time up to the
next iterate adds at most $Cmh$, since the active component velocities
are bounded.  Equation \eqref{eq:polynomial-step} makes these errors
smaller than the fixed entry margins, uniformly in $m$, giving entry
before $T_c$.  The same condition implies $D\ge A_I$ and $h\le a_I/m$.
Apply \cref{lem:discrete-continuation}.  Geometric loss decay and local
coercivity imply $(q_n,P_n,\theta_n)\to(1,0,0)$.
Exact exchange symmetry gives $q_{j,n}=P_n/M\to0$ for every redundant unit.
All constants are uniform over the fixed compact interval $I_\phi$.
\end{proof}

\section{Experimental and numerical validation}\label{app:validation}
We use the four-coordinate states $Y_\infty=(q_\infty,P_\infty,\theta_\infty,\phi)$ for the limit and $Y_D=(q_D,P_D,\theta_D,\psi_D)$ for the exact symmetric flow.  We compare these with training in the original parameters $(a_i,w_i)$.  Here $q$ and $\theta$ describe the selected neuron, while $P$ and $\psi$ describe the total coefficient and common direction of the redundant neurons.  Original-parameter training uses either the population loss or the empirical loss on independent samples $x_n\sim\mathcal N(0,I_2)$ with labels $y_n=[u^\top x_n]_+$.

Write $s_{\rm sel}$ and $s_{\rm red}$ for the selected and common redundant directions, and use fast time $\tau=Dt$.  In this appendix, $t_{\rm cap}$ denotes the first time the numerical selection criterion \eqref{eq:capture} is met, and $T>0$ is a fixed fast-time horizon.  For every fixed multiplicity, the theory predicts
\begin{equation}
\label{eq:validation-predictions}
\begin{aligned}
 t_{\rm cap}&=\Theta_m(D^{-1}),\\
 \int_{t_{\rm cap}}^\infty\!\|\dot s_{\rm red}(t)\|\,\dd t&=O_m(D^{-2}),\\
 \sup_{0\le\tau\le T}\|Y_D(\tau)-Y_\infty(\tau)\|_\infty&=O_{m,T}(D^{-2}).
\end{aligned}
\end{equation}

Here $\|\cdot\|_\infty$ is the maximum absolute coordinate.  Table~\ref{tab:single-protocol} lists the single-teacher configurations; the duplicate population and empirical studies use $\phi=\pi/4$.  Population direction lengths are integrated to the stated endpoint to approximate the infinite-time integral in \eqref{eq:validation-predictions}.  For each empirical width and seed, samples are nested across $N$.  Paired-bootstrap intervals resample the seed rows jointly across sample sizes, using 10,000 replicates.

\begin{table}[htbp]
\centering
\caption{Single-teacher experimental settings.  The population and empirical studies use duplicate initialization.}
\label{tab:single-protocol}
\small
\begin{tabular}{@{}p{0.17\linewidth}p{0.75\linewidth}@{}}
\toprule
Study & Configuration\\
\midrule
Single neuron & $q_\star=1$, $q_0=0.6$, $c_0=0$, $\varepsilon=0.05$;
$D\in\{2,4,8,16,32,64,128\}$\\[3pt]
Population & $m\in\{2,4,8,16,32\}$, $D\in\{8,16,32,64,128\}$;
endpoint $\tau=80$\\[3pt]
Empirical & $m\in\{2,4,8,16\}$, $D=32$,
$N\in\{512,1024,2048,4096,8192\}$;
20 coupled seeds, endpoint $\tau=60$\\
\bottomrule
\end{tabular}
\end{table}

For the exact trajectories $(q_D^\pm,c_D^\pm)$ in physical time, the single-unit errors in \cref{fig:diagnostics}(a) are
\begin{equation}
\label{eq:single-unit-errors}
\begin{aligned}
 e_+(D)&=\sup_{0\le\tau\le4}
 \bigl\|(q_D^+(\tau/D),c_D^+(\tau/D))-(Q^+(\tau),C^+(\tau))\bigr\|_2,\\
 e_-(D)&=\sup_{0\le\sigma\le2}
 \bigl|c_D^-(D\sigma)-C^-(\sigma)\bigr|.
\end{aligned}
\end{equation}
Here $(Q^+,C^+)$ and $C^-$ are the fast and slow limiting solutions of \eqref{eq:fast-plus-limit} and \eqref{eq:slow-minus-limit}.  Both intervals include zero.  The slow error measures alignment only.  Its coefficient approximation applies after the initial layer, on $\sigma\ge\sigma_0>0$, as in Proposition~\ref{prop:slow-minus}.  At zero the slow coefficient is $2\kappa(0)=1/\pi$, whereas $q_0=0.6$.  Supremum errors are approximated on 1,001 equally spaced time points, using $D\in\{4,8,16,32,64,128\}$.  The population and finite-sample sweeps use the selection criterion below, where $q_{\rm sel}=q$, $P_{\rm red}=P$, and $\theta_{\rm sel}=\theta$:
\begin{equation}
q_{\mathrm{sel}}\ge0.90,
\qquad
|P_{\mathrm{red}}|\le0.10,
\qquad
|\theta_{\mathrm{sel}}|\le0.10.
\label{eq:capture}
\end{equation}

The coefficient-weighted trajectory metric is defined in \eqref{eq:functionalstate}--\eqref{eq:sampleerror}.

\subsection{Agreement across dynamical descriptions}
\Cref{fig:trajectories} compares the four descriptions at $(m,D,N)=(8,32,8192)$.  Panel (c) plots the selected angle $|\theta|$ and the magnitude $|P\sin\psi|$ of the redundant coefficient-weighted direction perpendicular to the teacher.  Panel (d) uses the two terms in \eqref{eq:main-dissipation}, divided by $D$; for empirical training, their expectations are sample averages.  All four descriptions select the $+D$ unit and show the redundant coefficient decaying as the loss falls. The direction metrics distinguish this coefficient decay from alignment of the redundant feature.  The large-$D$ limit follows the coefficient-weighted coordinates with an $O(D^{-2})$ correction for redundant-direction motion.  The reduced and original-parameter systems agree to numerical precision across the population grid.  The full empirical system likewise agrees with its symmetry-reduced form, and permutation symmetry keeps the duplicate units identical.

\subsection{Scaling with imbalance magnitude and multiplicity}
Panels (a)--(b) of \cref{fig:scaling}, panel (b) of \cref{fig:diagnostics}, and \cref{tab:population} test the three predictions in \cref{eq:validation-predictions}.  The fitted exponents remain stable from $m=2$ through $m=32$.

\begin{table}[htbp]
\centering
\caption{Population scaling laws. All regressions use $D\in\{8,16,32,64,128\}$.  Here $\ell_{\rm red}$ denotes the post-selection path length of the common redundant direction, ``limit slope'' refers to the last error in \eqref{eq:validation-predictions}, and $R^2$ is the coefficient of determination of the log--log regression.}
\label{tab:population}
\small
\begin{tabular}{@{}rrrrrrr@{}}
\toprule
$m$ & $t_{\rm cap}$ slope & $R^2$ & $\ell_{\rm red}$ slope & $R^2$ & limit slope & $R^2$\\
\midrule
2  & -1.0019 & $>0.9999$ & -1.9991 & $>0.9999$ & -1.9981 & $>0.9999$\\
4  & -0.9982 & $>0.9999$ & -1.9965 & $>0.9999$ & -1.9983 & $>0.9999$\\
8  & -0.9967 & $>0.9999$ & -1.9945 & $>0.9999$ & -1.9926 & $>0.9999$\\
16 & -0.9951 & $>0.9999$ & -1.9964 & $>0.9999$ & -1.9941 & $>0.9999$\\
32 & -0.9964 & $>0.9999$ & -1.9907 & $>0.9999$ & -1.9946 & $>0.9999$\\
\bottomrule
\end{tabular}
\vspace{0.5em}

\begin{tabular}{@{}rrr@{}}
\toprule
$m$ & $Dt_{\rm cap}$ & $D^2\ell_{\rm red}$\\
\midrule
2  & 30.24 & $1.59\times10^{-2}$\\
4  & 24.16 & $3.78\times10^{-3}$\\
8  & 21.84 & $1.44\times10^{-3}$\\
16 & 20.82 & $6.40\times10^{-4}$\\
32 & 20.34 & $3.03\times10^{-4}$\\
\bottomrule
\end{tabular}
\end{table}

Both rescaled quantities remain bounded over the tested multiplicities, consistent with the fixed-$m$ predictions.

\subsection{Finite-sample trajectory scaling}
Table~\ref{tab:samples} shows an approximately linear relation between log trajectory error and log sample size at each width.  Every empirical run satisfies the selection criterion \eqref{eq:capture}.

\begin{table}[htbp]
\centering
\caption{Finite-sample full-trajectory scaling at $D=32$ using 20 coupled seeds. The error is defined in \eqref{eq:sampleerror}.}
\label{tab:samples}
\small
\begin{tabular}{@{}rrrrr@{}}
\toprule
$m$ & slope & $R^2$ & mean $\sqrt N\,\err_N$ & selection rate\\
\midrule
2  & -0.4908 & 0.9970 & 0.2316 & 20/20 at every $N$\\
4  & -0.4482 & 0.9589 & 0.2724 & 20/20 at every $N$\\
8  & -0.4718 & 0.9753 & 0.2356 & 20/20 at every $N$\\
16 & -0.4942 & 0.9858 & 0.2552 & 20/20 at every $N$\\
\bottomrule
\end{tabular}
\end{table}

The slopes are consistent with $N^{-1/2}$ over the tested fixed-dimensional configurations.  The weighting in \eqref{eq:functionalstate} keeps the metric well defined as the redundant coefficient approaches zero.

\subsection{Robustness and discrete training}
The angle, perturbation, multiplicity, and discrete-training studies use the selection criterion
\begin{equation}
 q_1>0.90,
 \qquad
 s_1^\top u>\cos(0.10),
 \qquad
 \|f_{\rm red}\|_{L^2}<0.10.
 \label{eq:capture-robustness}
\end{equation}
Table~\ref{tab:robust-protocol} lists the perturbation and discrete configurations.  For the near-duplicate study, starting from $q_i=1/m$, we draw a vector $z$ with independent standard-Gaussian entries, center and rescale it so that $\max_i|z_i|=1$, and set $q_i=(1/m)(1+\epsilon z_i)$.  Initial-angle perturbations and imbalance multipliers are sampled independently and uniformly from $[-\epsilon,\epsilon]$ and $[1-\epsilon,1+\epsilon]$, respectively, with the signs of $\delta_i$ preserved.  Each trajectory is integrated to $\tau=50$.

The discrete comparison records the maximum componentwise error in $(q_i,q_i s_i)_{i=1}^m$ between discrete and continuous original-parameter trajectories.

\begin{table}[htbp]
\centering
\caption{Perturbation and discrete-training configurations.  The first four studies end at $\tau=50$, and the direction-motion study ends at $\tau=80$.}
\label{tab:robust-protocol}
\small
\begin{tabular}{@{}p{0.19\linewidth}p{0.73\linewidth}@{}}
\toprule
Study & Configuration\\
\midrule
Initial angle & $D=32$, $m\in\{2,8,32\}$; eight equally spaced angles in $[0.77,0.90]$\\[3pt]
Width & $D=64$, $\phi=\pi/4$, $m\in\{2,4,8,16,32,64,128,256\}$\\[3pt]
Near duplicates & $(m,D,\phi)=(8,32,\pi/4)$;
$\epsilon\in\{10^{-4},10^{-3},10^{-2},0.03,0.05\}$, ten seeds each\\[3pt]
Discrete error & $(m,D,\phi)=(8,64,\pi/4)$;
$h\in\{0.025,0.05,0.1,0.2,0.4\}$\\[3pt]
Direction motion & $m=8$, $\phi=\pi/4$, $D\in\{16,24,32,48,64,96,128\}$;
$h\in\{0.2,0.1,0.05\}$, $\eta=h/D$\\
\bottomrule
\end{tabular}
\end{table}

\subsection{Post-selection direction motion under gradient descent}
The discrete-time analysis predicts $D^{-2}$ scaling of the redundant-direction path length.  For the configurations in \cref{tab:robust-protocol}, we sum the direction increments after the selection criterion \eqref{eq:capture-robustness} is first met.  Table~\ref{tab:discrete-transport} reports the fitted exponents and rescaled path lengths.

\begin{table}[htbp]
\centering
\caption{Post-selection path length of the common redundant direction under original-parameter gradient descent.  The exponent is fitted against $D$.}
\label{tab:discrete-transport}
\small
\begin{tabular}{@{}rrrr@{}}
\toprule
$h$ & fitted $D$ exponent & $R^2$ & mean $D^2\sum_n\|s_{n+1}-s_n\|$\\
\midrule
0.20 & -1.9953 & $>0.9999$ & $1.9029\times10^{-3}$\\
0.10 & -1.9952 & $>0.9999$ & $1.9294\times10^{-3}$\\
0.05 & -1.9975 & $>0.9999$ & $1.9305\times10^{-3}$\\
\bottomrule
\end{tabular}
\end{table}
For each fixed $h$, the rescaled path is nearly constant across the tested values of $D$, consistent with the predicted $D^{-2}$ scaling.

\begin{figure}[tbp]
\centering
\includegraphics[width=0.98\linewidth]{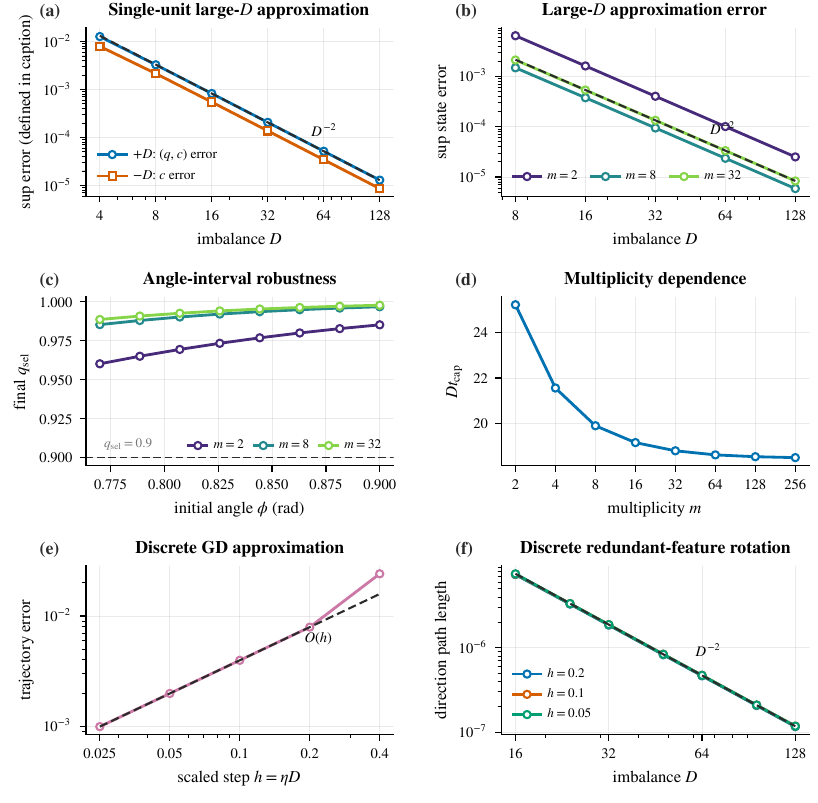}
\caption{Scaling and robustness results. (a) Single-unit errors \eqref{eq:single-unit-errors}: $e_+$ compares $(q,c)$ on $\tau\in[0,4]$, and $e_-$ compares only $c$ on $\sigma\in[0,2]$, both including zero. (b) Convergence of the exact reduced flow to its large-$D$ limit. (c) Selection across the tested angles and multiplicities. (d) Rescaled selection time across multiplicities. (e) Discrete-to-continuous trajectory error as a function of $h=\eta D$. (f) Post-selection motion of the redundant direction. Dashed lines show the predicted orders.}
\label{fig:diagnostics}
\end{figure}

\subsection{Two-teacher competition with nonduplicate students}
\label{sec:two-teacher}
Let $e_1,\ldots,e_{10}$ be the standard basis of $\R^{10}$.  We use $x\sim\mathcal N(0,I_{10})$ and
\[
 f_\star(x)=[u_1^\top x]_++0.8[u_2^\top x]_+,
 \qquad
 u_{1,2}=\cos(0.7)e_1\mp\sin(0.7)e_2.
\]
Eight students have initial directions
\[
 s_i(0)=\frac{\cos\theta_i e_1+\sin\theta_i e_2+0.12z_i}
 {\|\cos\theta_i e_1+\sin\theta_i e_2+0.12z_i\|},
 \qquad z_i\sim\mathcal N(0,I_8)\text{ in }\operatorname{span}(e_3,\ldots,e_{10}),
\]
where
$(\theta_1,\ldots,\theta_8)=(-0.45,-0.15,0.15,0.45,-1.1,1.1,-1.5,1.5)$.
Independently, $q_i(0)$ is uniform on $[0.15,0.25]$.  The smallest pairwise initial student angle over all 20 seeds is $0.349$ radians.  Thus the students are distinct and do not lie in the teacher plane.

Configuration A designates units $(1,4)$, configuration B designates $(2,3)$, and the same-side configuration designates $(1,2)$.  All three share the same initial coefficients and directions for each seed, and each is tested at $D\in\{16,32,64\}$.  The designated pair has imbalance $+D$, and all other neurons have imbalance $-D$.  Reconstruction through \eqref{eq:inverse-chart} changes only the parameter scales and preserves the predictor for every input.  Each neuron retains its own initial coefficient and direction across configurations.

We integrate the original-parameter population flow using the analytic Gaussian ReLU kernel, with DOP853 relative tolerance $10^{-9}$, absolute tolerance $10^{-11}$, and maximum fast-time step $2$.  Each of the 180 runs ($20$ seeds, $3$ imbalance magnitudes, $3$ assignments) ends at $\tau=Dt=4000$.  For the designated pair $(i_1,i_2)$ and teacher coefficients $(b_1,b_2)=(1,0.8)$, the endpoint success criterion is
\begin{equation}
\label{eq:two-teacher-success}
 \max_{k=1,2}\angle(s_{i_k},u_k)<0.05,
 \quad
 \max_{k=1,2}|q_{i_k}-b_k|<0.05,
 \quad
 \max_{j\notin\{i_1,i_2\}}|q_j|<0.01,
 \qquad(b_1,b_2)=(1,0.8).
\end{equation}
We evaluate this joint criterion at the common endpoint for all seeds.

\begin{table}[htbp]
\centering
\caption{Two-teacher endpoint results at $\tau=4000$.  The last column is the largest coefficient magnitude among neurons with negative imbalance over all students and seeds in the condition.}
\label{tab:two-teacher}
\small
\begin{tabular}{@{}rllr@{}}
\toprule
$D$ & neurons with $+D$ & success count & maximum redundant $|q_j|$\\
\midrule
16 & A: $(1,4)$ & 20/20 & $4.29\times10^{-7}$\\
16 & B: $(2,3)$ & 20/20 & $6.14\times10^{-4}$\\
16 & same side: $(1,2)$ & 15/20 & $8.97\times10^{-1}$\\
32 & A: $(1,4)$ & 20/20 & $4.30\times10^{-7}$\\
32 & B: $(2,3)$ & 20/20 & $5.57\times10^{-4}$\\
32 & same side: $(1,2)$ & 17/20 & $6.03\times10^{-1}$\\
64 & A: $(1,4)$ & 20/20 & $4.30\times10^{-7}$\\
64 & B: $(2,3)$ & 20/20 & $5.44\times10^{-4}$\\
64 & same side: $(1,2)$ & 17/20 & $5.64\times10^{-1}$\\
\bottomrule
\end{tabular}
\end{table}

Configurations A and B yield the specified teacher assignments across all seeds, as shown in \cref{fig:two-teacher,tab:two-teacher}.  The same-side configuration produces both assignments to distinct teachers and solutions that distribute the second feature among additional neurons.  In the latter case, both designated neurons can remain near the first teacher.  These outcomes exhibit the interaction between scale imbalance and initial direction in feature competition.

Table~\ref{tab:two-teacher-accuracy} reports alignment and numerical accuracy.  The tighter-tolerance comparison repeats all three configurations at $D=32$ for seeds 0 and 19, using relative and absolute tolerances $10^{-11}$ and $10^{-13}$.  We compare trajectories on a common grid containing zero and 600 logarithmically spaced points from $0.01$ to $4000$.

\begin{table}[htbp]
\centering
\caption{Alignment and numerical checks for the two-teacher experiment.  Values are maxima over the corresponding runs.}
\label{tab:two-teacher-accuracy}
\small
\begin{tabular}{@{}lr@{}}
\toprule
Quantity & Upper bound\\
\midrule
Assigned-teacher angle in A and B (radians) & $1.52\times10^{-4}$\\
Initial coefficient reconstruction error & $5.6\times10^{-17}$\\
Final imbalance drift & $1.2\times10^{-12}$\\
Coefficient and alignment change under tighter tolerances & $2.2\times10^{-8}$\\
\bottomrule
\end{tabular}
\end{table}

\subsection{Equal-scale control and single-student comparison}
The equal-scale control in \cref{fig:same-predictor} uses $m=8$, $D=32$, $q_i(0)=1/8$, $s_i(0)=s_0$, and $\phi=\pi/4$, as in the mixed-imbalance trajectory. All students have $\delta_i=+D$. The initial functional contribution of every student is therefore identical across the two configurations. Exact symmetry permits integration of the original input and output weights of one representative student, with all eight contributions included in the residual. We use DOP853 with relative tolerance $10^{-11}$, absolute tolerance $10^{-13}$, and maximum fast-time step $0.1$, recording the same grid on $0\le\tau\le60$ as the mixed configuration. Tightening both tolerances by a factor of ten changes the recorded parameters by less than $1.1\times10^{-13}$; the maximum imbalance drift is below $4.3\times10^{-14}$.

\Cref{fig:single-paths} displays the single-student comparison underlying \cref{thm:main-gap,cor:main-no-clock}. Both students start at $(q_0,c_0)=(0.6,0)$ with teacher coefficient $q_\star=1$. The phase curves use $D=16$ and stop at alignment $c=0.95$. The time ratio uses this same alignment threshold and the imbalance values in \cref{tab:single-protocol}. The points are the archived hitting-time measurements used to test the quadratic separation.

\end{document}